\documentclass[%
  onecolumn
   , hidempi
]{mpi2015-cscpreprint}

\usepackage[american]{babel}

\usepackage{graphicx}
\usepackage{amssymb}
\usepackage{amsthm}
\usepackage[mathscr]{eucal}
\usepackage{amsmath}
\usepackage{breqn}

\numberwithin{equation}{section}
\numberwithin{figure}{section}
\numberwithin{table}{section}

\usepackage{epstopdf}
\usepackage{import}
\usepackage{xifthen}
\usepackage{pdfpages}
\usepackage{placeins}
\usepackage{transparent}
\usepackage{cleveref}
\usepackage{sidecap}    
\usepackage{floatrow}
 \usepackage{multirow}
\usepackage{todonotes}
\usepackage{algorithm}
\usepackage[noend]{algpseudocode}

\makeatletter
\def\algbackskip{\hskip-\ALG@thistlm}
\makeatother

\usepackage{caption}
\usepackage{subcaption}
\usepackage{arydshln}

\definecolor{lightgray}{gray}{0.9}
\usepackage{color}
\definecolor{bluegreen}{rgb}{0.0, 0.87, 0.87}
\usepackage{afterpage}

\newtheorem{remark}{Remark}
\newtheorem{definition}{Definition}

\newtheorem{theorem}{Theorem}
\newcommand*{\bl}[1]{\mathbf{#1}}
\newcommand*{\bldt}[1]{\mathbf{\dot{#1}}}
\newcommand*{\blt}[1]{\mathbf{\tilde{#1}}}
\newcommand*{\bltdt}[1]{\mathbf{\dot{\tilde{#1}}}}
\newcommand*{\blh}[1]{\mathbf{\hat{#1}}}
\newcommand*{\blhdt}[1]{\mathbf{\dot{\hat{#1}}}}
\DeclareMathOperator{\sech}{sech}

\usepackage{mymacros}
\usepackage{todonotes}

\begin{document}
  

\title{Structure-preserving learning for multi-symplectic PDEs}

\author[$\ast$]{Süleyman Y\i ld\i z}
\affil[$\ast$]{Max Planck Institute for Dynamics of Complex Technical Systems, 39106 Magdeburg, Germany.\authorcr
	\email{yildiz@mpi-magdeburg.mpg.de}, \orcid{0000-0001-7904-605X}
}
  
\author[$\ast\ast$]{Pawan Goyal}
\affil[$\ast\ast$]{Max Planck Institute for Dynamics of Complex Technical Systems, 39106 Magdeburg, Germany.\authorcr
  \email{goyalp@mpi-magdeburg.mpg.de}, \orcid{0000-0003-3072-7780}
}
  
\author[$\dagger\ddagger$]{Peter Benner}
\affil[$\dagger$]{Max Planck Institute for Dynamics of Complex Technical Systems, 39106 Magdeburg, Germany.\authorcr
  \email{benner@mpi-magdeburg.mpg.de}, \orcid{0000-0003-3362-4103}
}
\affil[$\ddagger$]{Otto von Guericke University,  Universit\"atsplatz 2, 39106 Magdeburg, Germany\authorcr
  \email{peter.benner@ovgu.de} 
  \vspace{-0.5cm}
}
  
\shorttitle{Structure-preserving learning for multi-symplectic PDEs}
\shortauthor{S. Y\i ld\i z, P. Goyal, P. Benner}
\shortdate{}
  

  
\abstract{%
This paper presents an energy-preserving machine learning method for inferring reduced-order models (ROMs) by exploiting the multi-symplectic form of partial differential equations (PDEs). The vast majority of energy-preserving reduced-order methods use symplectic Galerkin projection to construct reduced-order Hamiltonian models by projecting the full models onto a symplectic subspace. However, symplectic projection requires the existence of fully discrete operators, and in many cases, such as black-box PDE solvers, these operators are inaccessible. In this work, we propose an energy-preserving machine learning method that can infer the dynamics of the given PDE using data only, so that the proposed framework does not depend on the fully discrete operators. In this context, the proposed method is non-intrusive. The proposed method is grey box in the sense that it requires only some basic knowledge of the multi-symplectic model at the partial differential equation level. We prove that the proposed method satisfy spatially discrete local energy conservation and preserves the multi-symplectic conservation laws.  We test our method on the linear wave equation, the Korteweg-de Vries equation and the Zakharov-Kuznetsov equation. We test the generalization of our learned models by testing them far outside the training time interval.
}

\novelty{
	\begin{enumerate}
		\item Model order reduction for the multi-symplectic PDEs is investigated.
		\item Non-intrusive model order reduction method for large-scale multi-symplectic PDEs is proposed.
		\item Several numerical examples tested to validate the analysis.
\end{enumerate}
} 
\maketitle

\section{Introduction}\label{sec:intro}

Partial differential equations (PDEs) are the main mathematical tool for describing many complex physical phenomena such as weather dynamics, chemical reaction dynamics, turbulent fluid flows, astrophysical plasmas and molecular dynamics. Numerical methods play an important role in solving PDEs, but the accuracy of their numerical solutions is highly dependent on the fine spatial and temporal discretization of the governing models, which makes the numerical solution approach cumbersome due to time-consuming simulation times and memory issues arising from large models. In many scenarios such as real-time simulation, structural design optimization, control and uncertainty quantification, fast simulation of these large models is required, which is often infeasible for high fidelity models. In this context, model order reduction (MOR) provides an efficient solution to this problem by constructing low-dimensional reduced-order models (ROMs). Over several decades, MOR techniques have been developed for various problems in science and engineering. A comprehensive overview of MOR techniques and their applications can be found in ~\cite{morBenGQetal21,morBenGQetal21a,morBenGQetal21b}.

Proper orthogonal decomposition (POD)~\cite{berkooz93} plays an important role in the many projection-based MOR approaches available and has been shown to be efficient and accurate in various applications \cite{morBenGQetal21b,Sirovich87,lumley67, Holmes12}. However, when dealing with particular systems such as Hamiltonian or multi-Hamiltonian systems, a direct application of projection-based MOR techniques falls short as it does not take into account the structure of the system, so that important properties of these systems, such as the preservation of physical laws, are lost. To address these issues, several structure-preserving MOR approaches have been developed. One such method is the symplectic MOR method proposed in~\cite{peng2016}, where the symplectic structure of the Hamiltonian equations is preserved by introducing the proper symplectic decomposition (PSD) method with Galerkin projection. The PSD is extended to non-orthonormal bases in~\cite{buchfink2019symplectic}. PSD ensures that the constructed reduced-order model is a Hamiltonian system, but it is restricted to canonical Hamiltonian systems. An extension of proper symplectic decomposition to non-canonical systems can be found in~\cite{miyatake2019structure}. The non-canonical extensions of the PSD method are limited to energy conservation. On the other hand, some recent studies have focused on global energy conservation using the multi-symplectic structure~\cite{yildiz23,uzunca2023global} and have shown that they preserve the global energy of the system. For a broader overview on structure-preserving model order reduction we refer to the review \cite{hesthaven2021structure}.

The structure-preserving ROMs mentioned above are intrusive and this implies that they require the availability of discrete differential operators. Access to these operators is very difficult in many scenarios, so non-intrusive MOR methods have recently become more attractive as they do not require the discrete operators. An important non-intrusive MOR strategy is called dynamic mode decomposition (DMD) \cite{schmid2010dynamic}, which learns linear systems from time-domain data by applying the technique associated with Koopman operator approximation. Another popular non-intrusive approach is the operator inference (OpInf) framework. The OpInf framework provides a way to learn reduced-dimensional representations of nonlinear problems. First, it is introduced for PDEs with low order polynomial nonlinear terms \cite{peherstorfer2016data}. In the following studies, the operator inference method is extended to general nonlinear systems by exploiting the lifting transformations in \cite{swischuk2020learning,qian2020lift}. The OpInf method is extended to nonlinear systems with known analytical expressions in \cite{benner2020operator}. More recently, the operator inference approach has been successfully adapted to Hamiltonian systems \cite{sharma2022hamiltonian}. The non-intrusive methods mentioned so far have mainly focused on learning the reduced-order operators. Some other studies have focused on modeling the Hamiltonian functions with neural networks and then using the Hamiltonian equations to learn the underlying systems \cite{chen2019symplectic,greydanus2019hamiltonian,finzi2020simplifying,offen2022symplectic}. Although these methods are quite successful in learning Hamiltonian systems, they focus on very low-dimensional systems, i.e. 3-4 dimensions. The idea of Hamiltonian neural networks and lifting transformations for low and high dimensional Hamiltonian systems has been successfully studied in \cite{yildizquad24}, the study shows that Hamiltonian systems can be transformed into simpler Hamiltonian systems such as cubic Hamiltonian systems, and their stability is investigated in \cite{goyal2023deep}.

In this work we are interested in constructing a structure-preserving MOR method inspired by the studies \cite{sharma2022hamiltonian,yildiz23}. Specifically, our goal is to obtain a non-intrusive MOR method suitable for multi-symplectic PDEs. While preserving the structure of the full model, we aim that the learned ROM inherits the properties of the full model, such as energy conservation. We assume the availability of knowledge at the PDE level. The proposed method is suitable for dealing with high-dimensional data arising from multi-symplectic PDEs. First, we project the high-dimensional data onto a low-dimensional basis via the basis obtained by the technique explained in \cite{yildiz23}, and learn the reduced operators from the projected data using an optimization problem. The proposed method can be seen as an extension of the Hamiltonian operator inference method proposed in \cite{sharma2022hamiltonian}, in the sense that it also preserves the multi-symplectic conservation law. With suitable discretizations, the proposed method provides a way to preserve the local quantities of multi-symplectic PDEs, which makes it different from the frameworks leading to structure-preserving models for canonical Hamiltonian systems.

The manuscript is structured as follows. \Cref{sec:FOM} describes the semi-discretization of the full multi-symplectic model. \Cref{sec:ROM} presents the non-intrusive structure-preserving method and the proposed operator inference problem. \Cref{sec:num} demonstrates the validity of the proposed approach with numerical results. Concluding remarks are given in \Cref{sec:conc}.

\section{Multi-symplectic Full-order Model} \label{sec:FOM}
One-dimensional multi-symplectic PDEs can be written in the following form:
\begin{equation}\label{eqn:Ms_PDEs}
	K z_t+L z_x=\nabla_z S(z),	\quad (x,t)\in \mathbb{R} \times \mathbb{R},
\end{equation}
where $ z(x, t)=[z_1(x, t),\ldots, z_d(x, t)]^\top \in\mathbb{R}^{d} $ is the vector of state variables, $ S(z) : \mathbb{R}^{d} \rightarrow \mathbb{R}$ is a smooth function and  $ K,  L \in \mathbb{R}^{d\times d} $ are skew-symmetric matrices. Multi-symplectic PDEs \eqref{eqn:Ms_PDEs} give rise to very useful invariants \cite{moore2003se}, which are as follows:
\begin{itemize}
	\item the multi-symplectic conservation law:
	\begin{equation}\label{eqn:mscons}
	\partial_t \omega+\partial_x\kappa=0, \quad\omega=dz\wedge K_+dz,\quad\kappa=dz\wedge L_+dz,
	\end{equation}  
	\item  the local momentum conservation law:
	\begin{equation*}
	I_t +G_x=0, \quad G=S(z)+z_t^T K_+z,\quad I=-z_x^T K_+z,
	\end{equation*}
	
	\item the local energy conservation law:
	\begin{equation*}
		E_t +F_x=0, \quad E=S(z)+z_x^TL_+z,\quad F=-z_t^TL_+z, \quad \text{and}
	\end{equation*}      
\end{itemize}   
where $K_+$ and $L_+$ satisfy
$$K=K_+-K_+^\top ,   \quad    L=L_+-L_+^\top ,$$ 
and $ \wedge $ denotes the wedge product. The global energy conservation law under periodic boundary conditions is defined as follows:
\begin{equation}\label{eqn:GCL}
\varepsilon_t=0,\quad \varepsilon(t)=\int_{\Omega}E(z(t))dx,
\end{equation}
where $E(z(t))=S(z)+z_x^TL_+z$.
\begin{definition}
Numerical schemes preserving the discrete version of the conservation of symplecticity \eqref{eqn:mscons} are said to be multi-symplectic.
\end{definition}
Next, we show a structure-preserving discretization of the one-dimensional multi-symplectic PDE \eqref{eqn:Ms_PDEs}. Two-dimensional structure-preserving space discretizations can easily be obtained using a similar procedure, see \cite{uzunca2023global}. 
To discretize the governing equation \eqref{eqn:Ms_PDEs}, we first introduce some notation such as the spatial domain $ \Gamma=\left[a,b\right] $, temporal node $ t_n=n\Delta t $, spatial node $ x_j=a+h(j-1) $, $ j=1,\ldots,N $, $n=0,1,\ldots $, where $ h=\tfrac{b-a}{N} $ is the spatial step size and $ \Delta t $ is the temporal step size. We denote the discretization of the function $ v(x,t) $ at the node $ (x_j,t_n) $ as $ v_j^n $. In addition, we define the following mean and difference operators:
\begin{align*}
	\delta_{t} v_j^n &:= \frac{v_j^{n+1} -v_j^n}{\Delta t}, &\delta_{t}^{1/2} v_j^n &:= \frac{v_j^{n+1} -v_j^{n-1}}{2\Delta t}, & \mu_{t} v_j^n &:= \frac{v_j^{n+1} +v_j^n}{2}, \\
	\delta_{x} v_j^n &:= \frac{v_{j+1}^n -v_{j}^n}{\Delta x}, &  \delta_{x}^{1/2} v_j^n &:= \frac{v_{j+1}^n -v_{j-1}^n}{2\Delta x}, & \mu_{x}v_j^n &:= \frac{v_{j+1}^{n} +v_j^n}{2}.
\end{align*}
The difference and average operators all commute with each other \cite{eidnes2020}, i.e,
$$\delta_{t}^{1/2}\delta_{x}v_j^n=\delta_{x}\delta_{t}^{1/2}v_j^n,\quad \delta_{t}\mu_{x}v_j^n=\mu_{x}\delta_{t}v_j^n,\quad \mu_{t}\delta_{x}^{1/2}v_j^n=\delta_{x}^{1/2}\mu_{t}v_j^n.$$
In addition, they satisfy the discrete Leibniz rule \cite{eidnes2020},
$$\delta_{t}(uv)_j^n=(\varepsilon u_j^{n+1}+(1-\varepsilon )u_j^n)\delta_{t}v_j^n+\delta_{t}u_j^n((1-\varepsilon) v_j^{n+1}+ \varepsilon v_j^{n}),\quad 0\le\varepsilon\le 1.$$
We denote the discretized vector of spatially discrete state variables as
$$ \bl{z}(t)=[z_{1,1}(t),\ldots,z_{1,N}(t),z_{2,1}(t),\ldots,z_{2,N}(t),\ldots ,z_{d,1}(t),\ldots,z_{d,N}(t)]^\top ,$$
where $ z_{d,j}(t)= z_d(x_{j},t) $, for $ j=1,2,\ldots,N $. For analysis purposes we define the vector $  \bl{z}_{m}(t) $, containing the $m^{\texttt{th}}$ node of each state as follows:
\begin{equation}\label{eqn:element}
	\bl{z}_{m}(t)=[{z}_{1,m}(t),{z}_{2,m}(t),\ldots,{z}_{d,m}(t)]^\top .
\end{equation}
To obtain a spatial discretization of the multi-symplectic PDE \eqref{eqn:Ms_PDEs}, we approximate the partial derivative $\partial_x$ with the central difference operator $\delta_{x}^{1/2}$. After spatial discretization, the semi-discrete equations reads as: 
\begin{equation}\label{eqn:semi-discrete-element}
	K\partial_t \bl z_m+ L\delta_{x}^{1/2} \bl z_m =\nabla_z S(\bl z_m), \quad m=1,\ldots,N.
\end{equation}
The semi-discrete equations \eqref{eqn:semi-discrete-element} can be written compactly as follows:
\begin{equation}\label{eqn:Ms-semi-discrete}
	\bl K\bldt z+ \bl L\bl{D}_x\bl z=\nabla_{\bl z}\bl S(\bl z),
\end{equation}
where $ D_x $ is a skew-symmetric matrix,   $ \bl S(\bl z) : \mathbb{R}^{ N\cdot d} \rightarrow \mathbb{R} $, $ \bl{K}=(K \otimes I_N)\in \mathbb{R}^{N\cdot d\times N\cdot d}$, $ \bl{L}=( L \otimes I_N)\in \mathbb{R}^{N\cdot d\times N\cdot d}$,  $ \bl{D}_x=( I_d\otimes D_x)\in \mathbb{R}^{N\cdot d\times N\cdot d},$  $ I_N\in\mathbb{R}^{N\times N},I_d\in\mathbb{R}^{d\times d} $ are identity matrices, and  $ \otimes $ denotes the Kronecker product.
For temporal discretization, we employ Kahan's method to obtain a linearly implicit global energy preserving (LIGEP) scheme. For details on temporal discretization , see \cite{eidnes2020}.

\section{Structure-preserving model reduction}\label{sec:ROM}
In this section, we first review the structure-preserving MOR method for multi-symplectic PDEs \eqref{eqn:Ms_PDEs} introduced in \cite{uzunca2023global,yildiz23}, then we propose a non-intrusive MOR method that is energy preserving. To construct a spatially discrete energy preserving ROM, we first define the snapshot matrix of the discrete state vectors as follows:
\begin{equation}\label{eqn:data_mat}
	\bl G=\left[\bl z(t_1),\ldots,\bl z(t_{N_t})\right]\in \mathbb{R}^{d\cdot N\times N_t}.
\end{equation}
where $  \bl z(t_i) \in \mathbb{R}^{d\cdot N} $ is the discrete state vector for $ i=1,2,\ldots,N_t $. Following \cite{uzunca2023global}, we obtain the POD basis with the following optimization problem:
\begin{equation}\label{eqn:opt}
	\min_{\substack{ \bl V }} \| \bl G - \bl V\bl V^\top \bl G \|_F, \quad \text{subject to}\quad \bl V^\top \bl K\bl V=\bl K_r,  \ \bl V^\top \bl L\bl V=\bl L_r,
\end{equation}
where $ \|\cdot\|_{F} $ denotes the Frobenius norm; $ \bl{K}_r=(K\otimes I_r)\in \mathbb{R}^{r\cdot d\times r\cdot d}$, $ \bl{L}_r=( L \otimes I_r)\in \mathbb{R}^{r\cdot d\times r\cdot d}$ with $ I_r \in \mathbb{R}^{r\times r} $ being the identity matrix. The solution to the above optimization problem can be computed using the POD basis of the following extended matrix
\begin{equation}\label{eqn:global_snap}
	\bl Z=[\bl Z_1,\ldots, \bl Z_d]\in \mathbb{R}^{N\times d\cdot N_t}
\end{equation}
where $ N_t $ is the number of time steps, and
$$ \bl Z_i=\left[[z_{i,1}(t_1),\ldots,z_{i,N}(t_1)]^\top,\ldots, [z_{i,1}(t_{N_t}),\ldots,z_{i,N}(t_{N_t})]^\top \right]\in \mathbb{R}^{N\times N_t}, \quad i=1,\ldots,d. $$ Let $ V\in \mathbb{R}^{N\times r} $ contain the POD basis of the global snapshot matrix \eqref{eqn:global_snap}, then we can construct the projection matrix as $ \bl{V}=( I_d\otimes V)\in \mathbb{R}^{d\cdot N\times d\cdot r}  $, where $ I_d\in \mathbb{R}^{d\times d} $ is the identity matrix.
The matrix $\bl V$ is an orthogonal matrix, i.e., $ \bl{V}^\top \bl{V}=I_{d\cdot r } $, so that the following properties hold:
\begin{align}\label{eq:reduced_relation}
	\bl V^\top \bl K\bl V &= \bl K_r,  & \bl V^\top \bl L\bl V&= \bl L_r, &
	\bl{V}^\top\bl{K}&= \bl{K}_r\bl{V}^\top, & \bl{V}^\top\bl{L}&=  \bl{L}_r\bl{V}^\top.
\end{align}
After obtaining the POD matrix, we construct our ROM by substituting the approximation $ \bl z\approx \blh{z}=\bl{V}\blt z$ in the semi-discrete model \eqref{eqn:Ms-semi-discrete}, which gives the following result:
\begin{equation}\label{FOM-approx}
	\bl{K} \bl{V}\bltdt z+\bl{L}\bl{D}_x\bl{V}\blt z=\nabla_{\bl z}\bl S( \bl{V}\blt z) + \cR(\blt z),
\end{equation}
where $\cR(\blt z)$ is the residual. Assuming $\cR(\blt z)$ is orthonormal to $\bl V$, the ROM can be written as follows:
\begin{equation}\label{eqn:Ms-proj}
	\bl{V}^\top\bl{K} \bl{V}\bltdt z+\bl{V}^\top\bl{L}\bl{D}_x\bl{V}\blt z=\bl{V}^\top\nabla_{\bl z}\bl S( \bl{V}\blt z).	
\end{equation} 
Using the properties in \eqref{eq:reduced_relation}, we rewrite the above equations as follows:
\begin{equation}\label{MS-sd-rom}
	\bl{K}_r\bltdt z+\bl{L}_r\blt{D}_x\blt z=\nabla_{\blt z}\bl S(\bl{V}\blt z),
\end{equation} 
where $ \blt{D}_x= \bl V^\top\bl{D}_x \bl V=( I_d\otimes \tilde{D}_x) $.
Using the notation introduced in \eqref{eqn:element}, we can write the semi-discrete equations \eqref{MS-sd-rom} as
\begin{equation}\label{eqn:MS-sd-rom-element}
	K\bltdt z_m+ L(\blt{D}_x \blt z)_m =\nabla_{\blt z}\bl S((\bl{V}\blt z)_m), \quad m=1,\ldots,r.
\end{equation} 
\begin{theorem}\label{th:semi-discrete-cons-law}
The spatially-discrete equation \eqref{eqn:Ms-proj} yields a semi-discrete multi-symplectic conservation law
\begin{equation*}
	\partial_t \omega+\kappa=0, \quad\omega=\bl{d}\blt{z}\wedge\bl{K}_r \bl{d}\blt z,\quad\kappa=2\bl{d}\blt{z}\wedge\bl{L}_r\blt{D}_x\bl{d}\blt z.
\end{equation*}
\end{theorem}
\begin{proof}
We start with the semi-discrete form of the governing equations 
\begin{equation*}
	\bl{K}_r\bltdt z+\bl{L}_r\blt{D}_x\blt z=\bl{V}^\top\nabla_{\bl z}\bl S(\bl{V}\blt z).
\end{equation*}	
Then the associated variational equation becomes
\begin{equation*}
	\bl{K}_r \bl{d}\bltdt z+\bl{L}_r\blt{D}_x\bl{d}\blt z=S_{\blt{z}\blt{z}}(\blt{z})\bl{d}\blt{z},
\end{equation*}
where $S_{\blt{z}\blt{z}}(\blt{z})=\bl{V}^\top\bl S_{zz}(\bl{V}\blt z)\bl{V}$ is the Hessian matrix.
Since $S_{\blt{z}\blt{z}}$ is symmetric, taking the wedge product of the above equation with $\bl{d}\blt{z}$, we end up with
\begin{equation*}
	\bl{d}\blt{z}\wedge\bl{K}_r \bl{d}\bltdt z+\bl{d}\blt{z}\wedge\bl{L}_r\blt{D}_x\bl{d}\blt z=0
\end{equation*}
which implies
\begin{align*}
	0&=\partial_t (\bl{d}\blt{z}\wedge\bl{K}_r \bl{d}\blt z)-\bl{d}\bltdt{z}\wedge\bl{K}_r \bl{d}\blt z+\bl{d}\blt{z}\wedge\bl{L}_r\blt{D}_x\bl{d}\blt z\\
	&=\partial_t (\bl{d}\blt{z}\wedge\bl{K}_r \bl{d}\blt z)+\bl{K}_r\bl{d}\bltdt{z}\wedge \bl{d}\blt z+\bl{d}\blt{z}\wedge\bl{L}_r\blt{D}_x\bl{d}\blt z\\
	&=\partial_t (\bl{d}\blt{z}\wedge\bl{K}_r \bl{d}\blt z)-\bl{d}\blt{z}\wedge \bl{K}_r \bl{d}\bltdt z+\bl{d}\blt{z}\wedge\bl{L}_r\blt{D}_x\bl{d}\blt z\\
	&=\partial_t (\bl{d}\blt{z}\wedge\bl{K}_r \bl{d}\blt z)+2\bl{d}\blt{z}\wedge\bl{L}_r\blt{D}_x\bl{d}\blt z.\\
\end{align*}
Let us denote $\bl{A}_r=\blt{D}_x\bl{L}_r$. Since  $\blt{D}_x\bl{L}_r=\bl{L}_r\blt{D}_x$, $\bl{A}_r$ is symmetric  which implies conservation of total symplecticity over time.
\end{proof}

\begin{theorem}\label{th:semi-discrete-local-energy}
	The spatially-discrete equation \eqref{eqn:MS-sd-rom-element} yields a semi-discrete energy conservation law
	\begin{equation*}
		\partial_tE_m+F_{m}=0, \quad m=1,\ldots,r
	\end{equation*}
	with
	\begin{align*}
		E_m&=S( \blh{z}_m)-\frac{1}{2}\langle  \blt{z}_{m},L(\blt{D}_x \blt z)_m\rangle,\\
		F_{m}&=\sum_{k=1}^{r}(\tilde{D}_x)_{m,k}\langle  \blt{z}_{m},L(\blt{D}_x \bltdt z)_m\rangle.
	\end{align*}
\end{theorem}
\begin{proof}
	Taking the inner product of both sides \eqref{eqn:MS-sd-rom-element} with $ \bltdt z$, we  have
	$$\langle \bltdt{z}_{m},K\bltdt z_m \rangle+ \langle \bltdt{z}_{m},L(\blt{D}_x \blt z)_m\rangle =\langle \bltdt{z}_{m},\nabla_{\blt z}\bl S((\bl{V}\blt z)_m)\rangle, \quad m=1,\ldots,r.$$
	Since $K$ is skew-symmetric, we have $ \langle \blhdt{z}_{m},K\blhdt z_{m}\rangle=0 $. Additionally, using the identity $ \partial_t S(\blh z_m)=\langle \blhdt{z}_{m},\nabla_zS(\blh z_m)\rangle $, we have
	\begin{align*}
		0&=\partial_t\left[S(\blh z_m)- \langle \blt{z}_{m},L(\blt{D}_x \blt z)_m\rangle \right] +\langle \bltdt{z}_{m},L(\blt{D}_x \blt z)_m\rangle, \\
		&=\partial_t\left[S(\blh z_m)- \langle \blt{z}_{m},L(\blt{D}_x \blt z)_m\rangle \right] +\sum_{k=1}^{r}(\tilde{D}_x)_{m,k}\langle  \blt{z}_{m},L(\blt{D}_x \bltdt z)_m\rangle, \quad m=1,\ldots,r,
	\end{align*}
	which concludes the proof. 
\end{proof}

\Cref{th:semi-discrete-cons-law,th:semi-discrete-local-energy} shows that the ROM is a multi-symplectic scheme and it preserves the semi-discrete local energy conservation laws. These properties are important in the sense that the desired solutions of the full models are obtained by time integration of the proposed ROM, and preserving these properties has been shown to lead to robust models \cite{frank2006linear,ascher2004multisymplectic}. Preserving multi-symplectic formulations in the fully discrete form is not a straightforward task, as preserving the multi-symplectic conservation law does not necessarily mean that all the invariants will be preserved, e.g. the Preissmann box scheme preserves the multi-symplectic conservation laws but does not preserve the local energies for nonlinear Hamiltonian systems \cite{moore2003}. One way to construct a fully discrete multi-symplectic formulation from semi-discrete equations is to use symplectic integrators. Consider the semi-discrete equation \eqref{eqn:Ms-semi-discrete}, discretizing the time domain with a symplectic integrator, e.g. with the implicit midpoint rule, yields the following
\begin{equation}\label{eqn: midpoint}
		\delta_{t}\bl{K}_r\blt z^{n}+\bl{L}_r\blt{D}_x \mu_{t}\blt z^{n}=\nabla_{\blt z}\bl S(\mu_{t}\bl{V}\blt z^{n}).
\end{equation}
The fully-discrete scheme \eqref{eqn: midpoint} preserve the multi-symplectic conservation law. Consider the variational equations associated with \eqref{eqn: midpoint}
\begin{align*}
		\delta_{t}\bl{K}_r\bl{d}\blt z^{n}+\bl{L}_r\blt{D}_x \bl{d}(\mu_{t}\blt z^{n})=\bl{S}_{\blt{z}\blt{z}}(\mu_{t}\blt{z}^n)\bl{d}(\mu_{t}\blt z^{n})
\end{align*}
taking the wedge product of both sides with $\bl{d}\mu_{t}\blt{z}^n$  of the above equation again implies the preservation of the total symplectic form in time as it reads as 
\begin{align*}
	0&=\bl{d}\mu_{t}\blt{z}^n\wedge\delta_{t}\bl{K}_r\bl{d}\blt z^{n},\\
	 &=\frac{1}{2}\delta_{t}\left[\bl{d}\blt{z}^n\wedge\bl{K}_r\bl{d}\blt z^{n}\right].
\end{align*}
An alternative is using Kahan's method \cite{eidnes2020} that has shown to satisfy the global energy preservation laws. For linear multi-symplectic PDEs, Kahan's method coincides with the midpoint rule \eqref{eqn: midpoint}.
We describe a non-intrusive way to recover the multi-symplectic discretizations. Our key ingredient is, as explained in the previous subsection, is that for any skew-symmetric matrix, multi-symplectic laws can be preserved. 
First let us define the following time derivative data matrix
\begin{equation*}
	\bl G=\left[\bldt z^1,\ldots,\bldt z^{N_t}\right]\in \mathbb{R}^{d\cdot N\times N_t}.
\end{equation*}
 Furthermore, we define the following reduced snapshot matrix and time derivative matrix
\begin{equation}\label{eqn:reduced-data-mat}
	\blt{G}=\bl{V}^T\bl{G}, \quad \bltdt{G}=\bl{V}^T\bldt{G}
\end{equation}
Assuming the basic knowledge of the Hamiltonian $S$, we define the following nonlinear forcing snapshot matrix and its projection as follows:
\begin{equation}
	\bl{F}=\left[\grad_{\bl{z}}\bl{S}(\bl{z}^1),\ldots,\grad_{\bl{z}}\bl{S}(\bl{z}^{N_t})\right], \quad, \blt{F}=\bl{V}^T\bl{F}.
\end{equation}
The reduced-order operators can be learned by standard operator inference in the continuous-time setting by solving the following optimization problem:
\begin{equation}\label{eqn:MS-cont-opinf}
		\min_{\blt{D}_x=-\blt{D}_x^T}\left\|\bl{K}_r\bltdt G+\bl{L}_r\blt{D}_x\blt G-\blt{F}(\blh{z}^n)\right\|_F.
\end{equation}

Learning the reduced-order operators via the above problem can be problematic in the sense that the time derivatives have to be approximated. Instead, one can learn the reduced-order operators using time-discrete equations, e.g. consider \eqref{eqn: midpoint}, in which case the reduced-order operators can be obtained by solving the following problem:
\begin{equation}\label{eqn:MS-OpInf-discrete}
	\min_{\blt{D}_x=-\blt{D}_x^T}\left\|\delta_{t}\bl{K}_r\blt G^n+\bl{L}_r\blt{D}_x \mu_{t}\blt G^n-\mu_{t}\blt{F}^n\right\|_F.
\end{equation}
where $\bl G^n=\left[\bl z^1,\ldots,\bl z^{N_t-1}\right]$, and $\bl G^{n+1}=\left[\bl z^2,\ldots,\bl z^{N_t} \right]$, i.e. the upper index notation implies standard and time-shifted data matrices in the matrix case. Learning a model via the above optimization problem \eqref{eqn:MS-OpInf-discrete} might be problematic because the auxiliary variables need to be eliminated. So it might be useful to know the final form of the equations in which the auxiliary variables are eliminated and learn the reduced model via the final time-discrete form of the equations.
\begin{remark}
	Since \eqref{eqn: midpoint} is a multi-symplectic scheme for any skew-symmetric matrix $\blt{D}_x$, the learned model through the optimization problem \eqref{eqn:MS-OpInf-discrete} results in a multi-symplectic model. The proposed model can be applied to time discrete or continuous equations and different strategies can be used, such as learning the reduce operators in \eqref{MS-sd-rom} and then finding a suitable time integration method.
\end{remark}

\section{Numerical Results}\label{sec:num}	
In this section, we study the performance of multi-symplectic OpInf in terms of discrete energy conservation and accuracy outside the training interval, in order to assess the generality of the method. We test our method on three different examples, namely the wave equation, the Korteweg--de Vries (KdV) equation and the Zakharov-Kuznetsov (ZK) equation. We employ Kahan's method \cite{eidnes2020} to \eqref{MS-sd-rom} to obtain the snapshot data. Since the wave equation is linear, Kahan's method coincides with the midpoint rule \eqref{eqn: midpoint} in time. 

To solve the optimization problem arising from the OpInf framework, we use \texttt{PyTorch} with the Adam algorithm \cite{kingma2014adam} in combination with a \texttt{ReduceLROnPlateau} scheduler with the \texttt{PyTorch Lightning} \cite{lightning} module. We set the initial learning rate of the scheduler to $10^{-2}$, the minimum learning rate to $10^{-4}$, and the threshold to $10^{-6}$.

We determine the performance of the proposed method by the convergence to the conserved energy in the following relative error:
\begin{equation}\label{eqn:rel_ener_err}
	\dfrac{|\bar{\mathcal{E}}(t_n)-\bar{\mathcal{E}_r}(t_n)|}{|\bar{\mathcal{E}}(t_n) |},
\end{equation}
where $\bar{\mathcal{E}}(t_n)$ denotes the global energy \eqref{eqn:GCL} preserved by the ground truth model and $\bar{\mathcal{E}_r}(t_n)$ denotes the global energy preserved \eqref{eqn:GCL} by the proposed non-intrusive method.

\subsection{Linear wave equation}\label{subsec:wave}
Our first example is a linear wave equation of the form  
\begin{equation}\label{eqn:wave}
	u_{tt} =c^2u_{xx},
\end{equation}
where the constant $c$ is the speed of the wave. In this example, we set the wave speed to $c=1$ and examine the performance of the proposed method on a single parameter.  The wave equation \eqref{eqn:wave} is an example of a multi-symplectic PDE \eqref{eqn:Ms_PDEs}.  To define the multi-symplectic form of the wave equation, we first define the following auxiliary variables: 
\begin{equation*}
	\begin{split}
		v=u_t, \quad w=u_x
	\end{split}.
\end{equation*}
Using the above relations, we can rewrite the wave equation \eqref{eqn:wave} in the multi-symplectic form \eqref{eqn:Ms_PDEs} with
\begin{equation*}
	z=
	\begin{bmatrix}
		u\\
		v\\
		w
	\end{bmatrix},\qquad	
	K=
	\begin{bmatrix}
		0  &-1 & 0  \\
		1  & 0 & 0  \\
		0  & 0 & 0 
	\end{bmatrix},\qquad
	L=
	\begin{bmatrix}
		0  & 0  & 1 \\
		0  & 0  & 0 \\
		-1  & 0  & 0   
	\end{bmatrix}
\end{equation*}	
and the Hamiltonian $S(z)=\frac{1}{2}(v^2-w^2)$. Applying the LIGEP method \cite{yildiz23}, the discrete wave equation reads as follows:
\begin{equation}\label{eqn:wave-LIGEP-couple}
	\begin{split}
		-\delta_t v_j^n+ \delta_x^{1/2}\mu_tw_j^n &=0,\\
		\delta_tu_j^n &= \mu_tv_j^n,\\
		-\delta_x^{1/2} \mu_t u_j^n &=-\mu_tw_j^n.
	\end{split}
\end{equation}
After removing the auxiliary variables, the fully discrete equations can be equivalently expressed as follows:
\begin{align}\label{eqn:wave_LIGEP}
	\delta_t^2u_j^n - \mu_t^2 \left(\delta_x^{1/2}\right)^2 u^{n}_j=0,
\end{align}
where $\left(\delta_x^{1/2}\right)^2=\delta_x^{1/2}\delta_x^{1/2}$. The polarized discrete energy conserved by the full model \eqref{eqn:wave_LIGEP} can be expressed as
\begin{equation}\label{eqn:wave-fom-pol-ener}
	\begin{split}
		\bar{\mathcal{E}}(t_n) = & \, \frac{\Delta x}{6}\sum_{j=1}^{N}\Big(2(\delta_x^{1/2}u^n_j)(\delta_x^{1/2} u^{n+1}_j) + (\delta_x^{1/2}u^n_j)^2+2v^n_jv^{n+1}_j + (v^n_j)^2\Big),
	\end{split}
\end{equation}
where $ v^n_j=\delta_tu^n_j-\frac{\Delta t}{2}\mu_t(\delta_x^{1/2})^{2}u^n_j .$
To construct a reduced multi-symplectic model, we define the following extended snapshot matrix:
$$\bl Z=\left[ \bl u(t_1),\ldots,\bl u(t_{N_t}),\bl v(t_1),\ldots,\bl v(t_{N_t}),\bl w(t_1),\ldots,\bl w(t_{N_t})\right]\in \mathbb{R}^{N\times 3 N_t},$$
where the states $ \bl v $ and $ \bl w $ are obtained from the relations \eqref{eqn:wave-LIGEP-couple}, which are $ v^n_j=\delta_tu^n_j-\frac{\Delta t}{2}\mu_t(\delta_x^{1/2})^{2}u^n_j $ and $ \delta_x^{1/2}  u_j^n =w_j^n .$  
By utilizing the extended snapshot matrix defined above to determine the POD basis, the fully-discrete multi-symplectic model for the wave equation using the LIGEP method can be written as follows:
\begin{equation}\label{eqn:wave_LIGEP_ROM_el}
	\begin{split}
		-\delta_t \blt v^n+ \mu_t \tilde{D}_x\blt w^n &=0,\\
		\delta_t \blt u^n &= \mu_t \blt v^n,\\
		-\mu_t \tilde{D}_x \blt u^n &=-\mu_t \blt w^n,
	\end{split}
\end{equation}
where $ \tilde{D}_x=V^\top D_xV $. As we have shown previously in \Cref{sec:ROM}, since the considered wave equation is linear, the LIGEP method coincides with the mid-point rule and satisfies the multi-symplectic conservation law \eqref{eqn:mscons}.
Removing the auxiliary variables in \eqref{eqn:wave_LIGEP_ROM_el}, the fully-discrete multi-symplectic model can be expressed as
\begin{equation}\label{eqn:wave_LIGEP_ROM}
	\delta_t^2\blt u^n - \mu_t^2 \tilde{D}_x^2 \blt u^{n}=0.
\end{equation}
The ROM \eqref{eqn:wave_LIGEP_ROM} conserves the following approximated polarized energy
\begin{equation}\label{eqn:wave-rom-pol-ener}
	\begin{split}
		\bar{\mathcal{E}_r}(t_n) = & \, \frac{\Delta x}{6}\sum_{j=1}^{N}\Big(2( V \tilde{D}_x \blt u^n)_j( V \tilde{D}_x \blt u^{n+1})_j + ( V \tilde{D}_x \blt u^n)_j^2+2\blh v^n_j \blh v^{n+1}_j + (\blh v^n_j)^2\Big),
	\end{split}
\end{equation}
where 
\begin{equation*}
	\blh v^n_j=\delta_t \blh u^n_j-\frac{\Delta t}{2}\mu_t\left(( V \tilde{D}_x)^{2}\blt u^n\right)_j.
\end{equation*}
Finally, the multi-symplectic operator inference model can be derived through the following optimization problem:
\begin{equation}\label{eqn:wave-OpInf-opt}
	\min_{\tilde{D}_x=-\tilde{D}_x^T}\left\|\delta_t^2\blt U^n - \mu_t^2 \tilde{D}_x^2 \blt U^n\right\|_F,
\end{equation}
where $\bl U^n=\left[ \bl u(t_1),\ldots,\bl u(t_{N_t})\right]$ and $\blt{U}^n=V^T\bl{U} \in \mathbb{R}^{r\times Nt}$.
We use the following initial condition:
\begin{align*}
	u_t(x,0)&=0,\\
	u(x,0)&=\sech(x),
\end{align*}
with periodic boundary conditions.
We discretize the spatial domain $\Omega=[-5,5] $ with a spatial step-size $\Delta x=0.0196$, giving a full model dimension of $N=512$, and we set the temporal step size to $\Delta t=0.1$ to construct the extended snapshot matrix \eqref{eqn:global_snap}. We set the reduced dimension $r=16$ to learn a non-intrusive reduced-order model.  
To test the generality of our proposed method, after obtaining the coefficients of the ROM, we train the proposed model with the snapshots simulated up to the final time $T=5$ and then simulate it until $T=20$. In \Cref{fig:wave-rcom}, we compare the first three reduced coefficients $\tilde{u}_i(t)$ obtained by the ground truth model \eqref{eqn:wave_LIGEP} with the trajectories of the multi-symplectic OpInf model $\breve{u}_i(t)$ obtained by \eqref{eqn:wave-OpInf-opt} in the reduced dimension as representatives. We note that the absolute error between the other coefficients are of the same magnitude. To highlight the training interval, we split the training and testing intervals with a vertical line in \Cref{fig:wave-rcom}.

In \Cref{fig:wave-OpInf-ener}, we plot the energy comparison of the discrete wave equation using the LIGEP method \eqref{eqn:wave_LIGEP} and the multi-symplectic OpInf method obtained through \eqref{eqn:wave-OpInf-opt}. We compare the discrete energies in terms of the relative error \eqref{eqn:rel_ener_err} between the ground truth model global energy \eqref{eqn:wave-fom-pol-ener} and the learned OpInf method global energy \eqref{eqn:wave-rom-pol-ener}. The figure shows that the OpInf accurately captures the global energy \eqref{eqn:wave-fom-pol-ener}.

\begin{figure}[h]
	\centering
	\begin{subfigure}{0.328\textwidth}
		\includegraphics[width=1\linewidth]{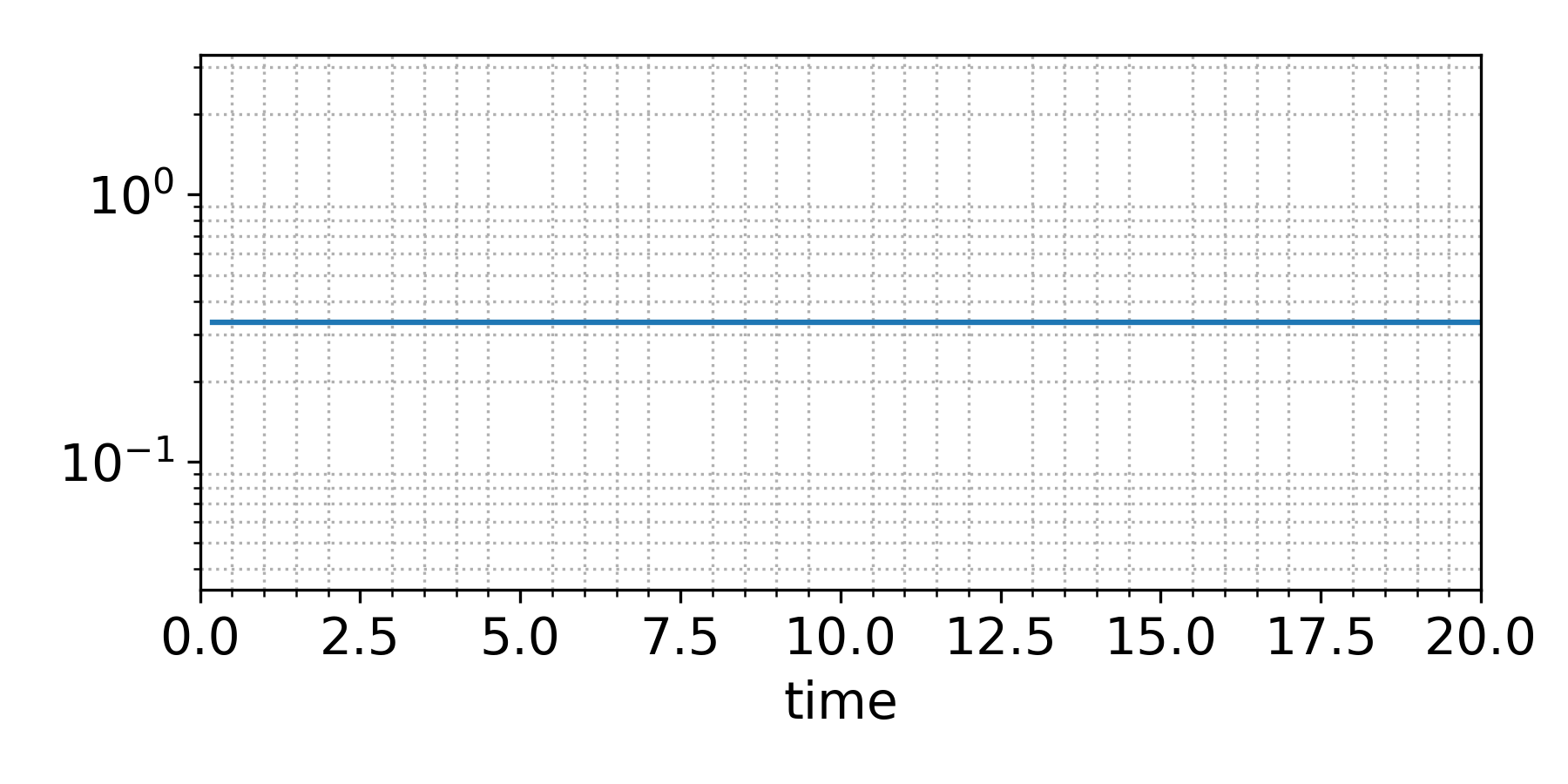}
		\caption{Ground truth energy}	
	\end{subfigure}
	\begin{subfigure}{0.328\textwidth}
		\includegraphics[width=1\linewidth]{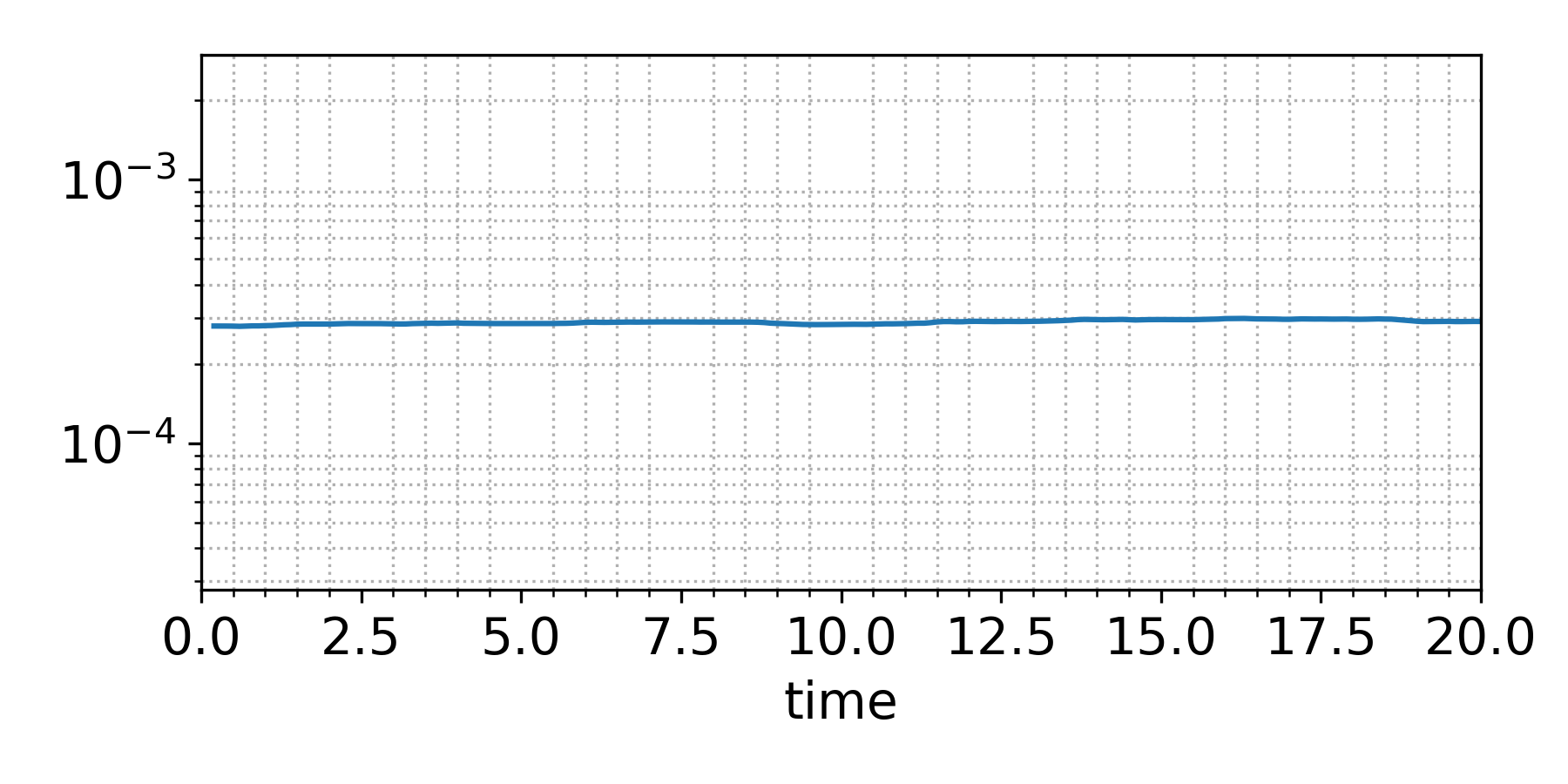} 
		\caption{Relative error}
	\end{subfigure}

	\caption{Linear wave equation: discrete energy comparison between full and reduced-order models.}
	\label{fig:wave-OpInf-ener}
\end{figure}

\begin{figure}[h]
	\centering
	\begin{subfigure}{0.328\textwidth}
	\includegraphics[width=1\linewidth]{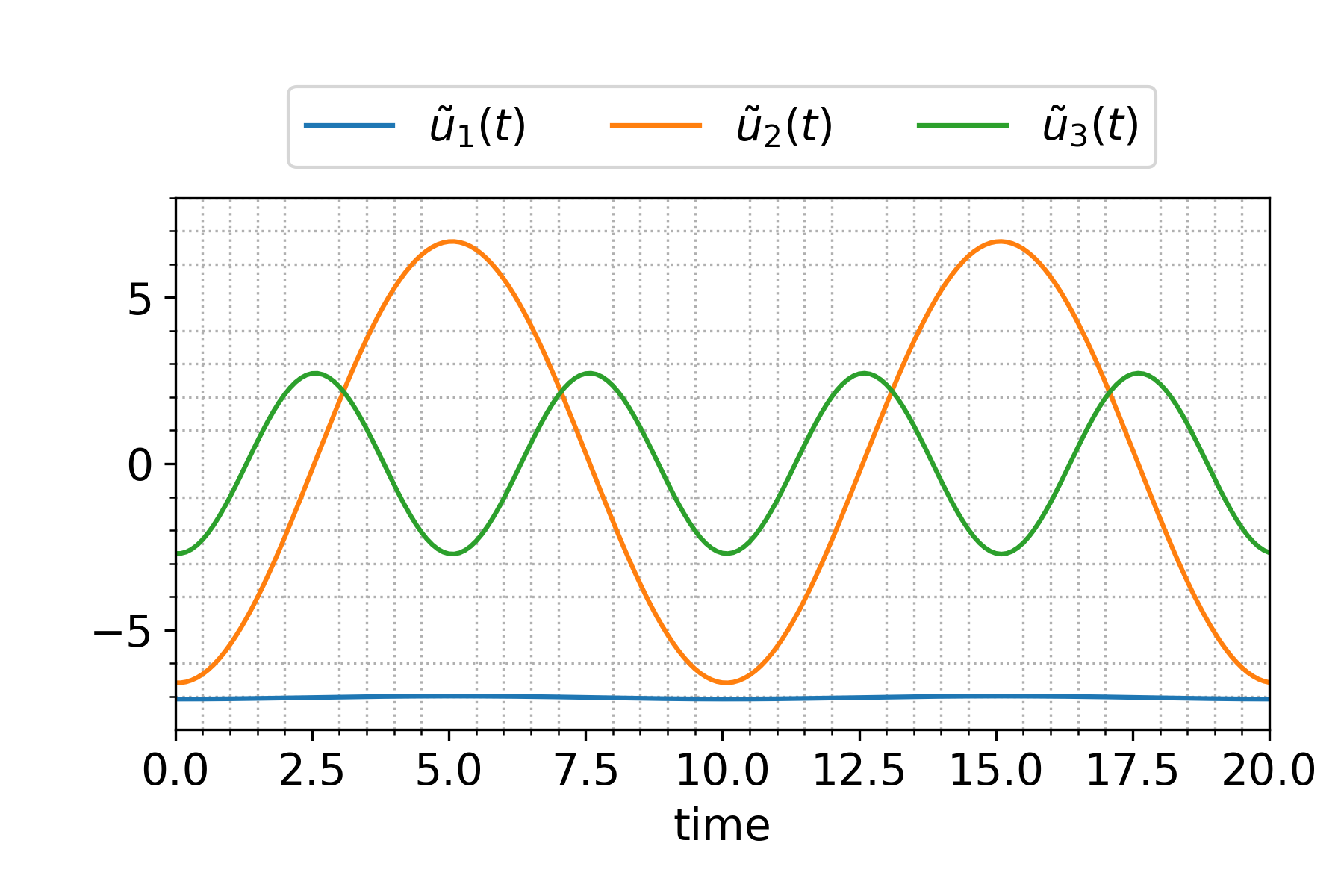}
	\caption{Ground truth}	
	\end{subfigure}
	\begin{subfigure}{0.328\textwidth}
	\includegraphics[width=1\linewidth]{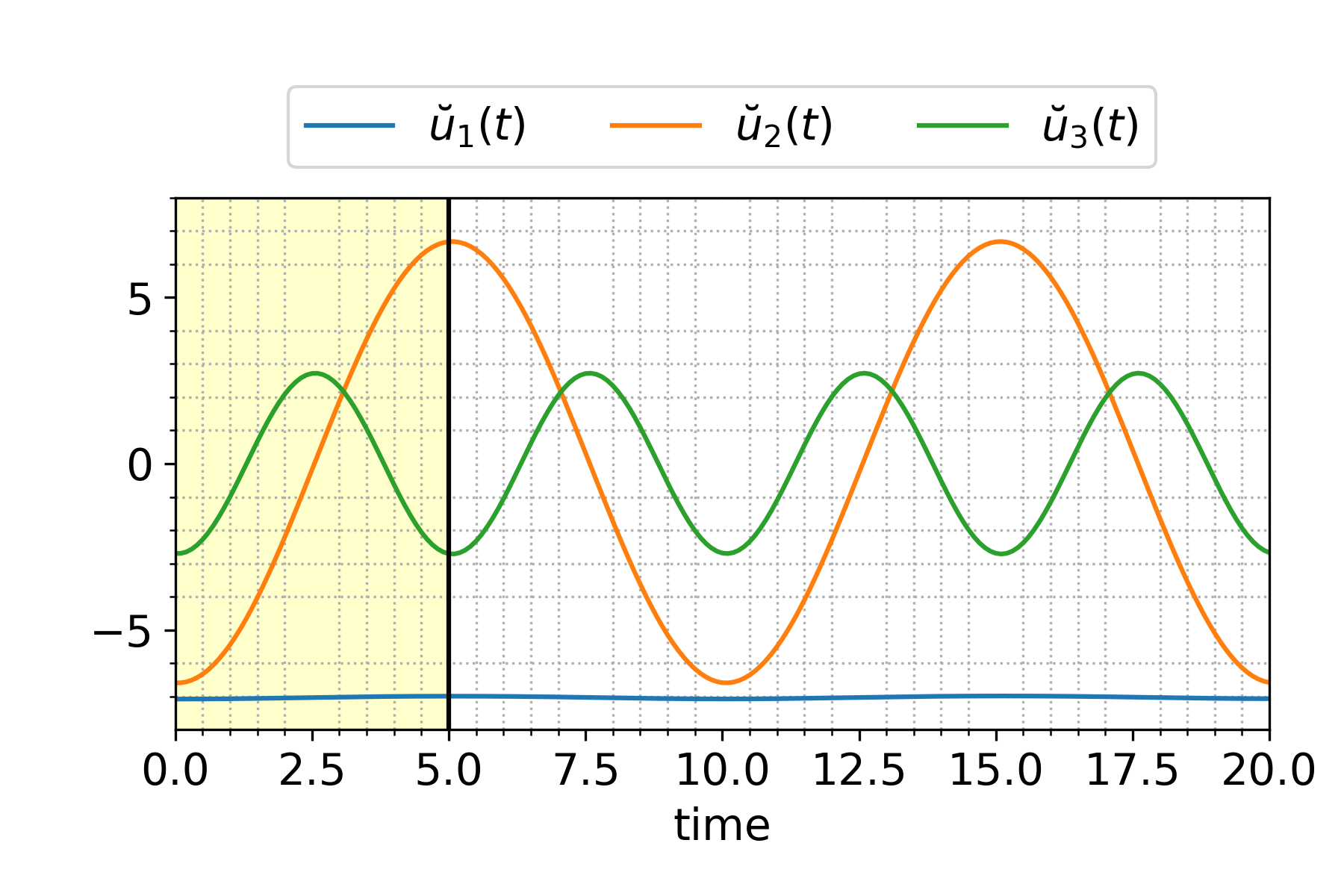} 
    \caption{OpInf}
	\end{subfigure}
	\begin{subfigure}{0.328\textwidth}
	\includegraphics[width=1\linewidth]{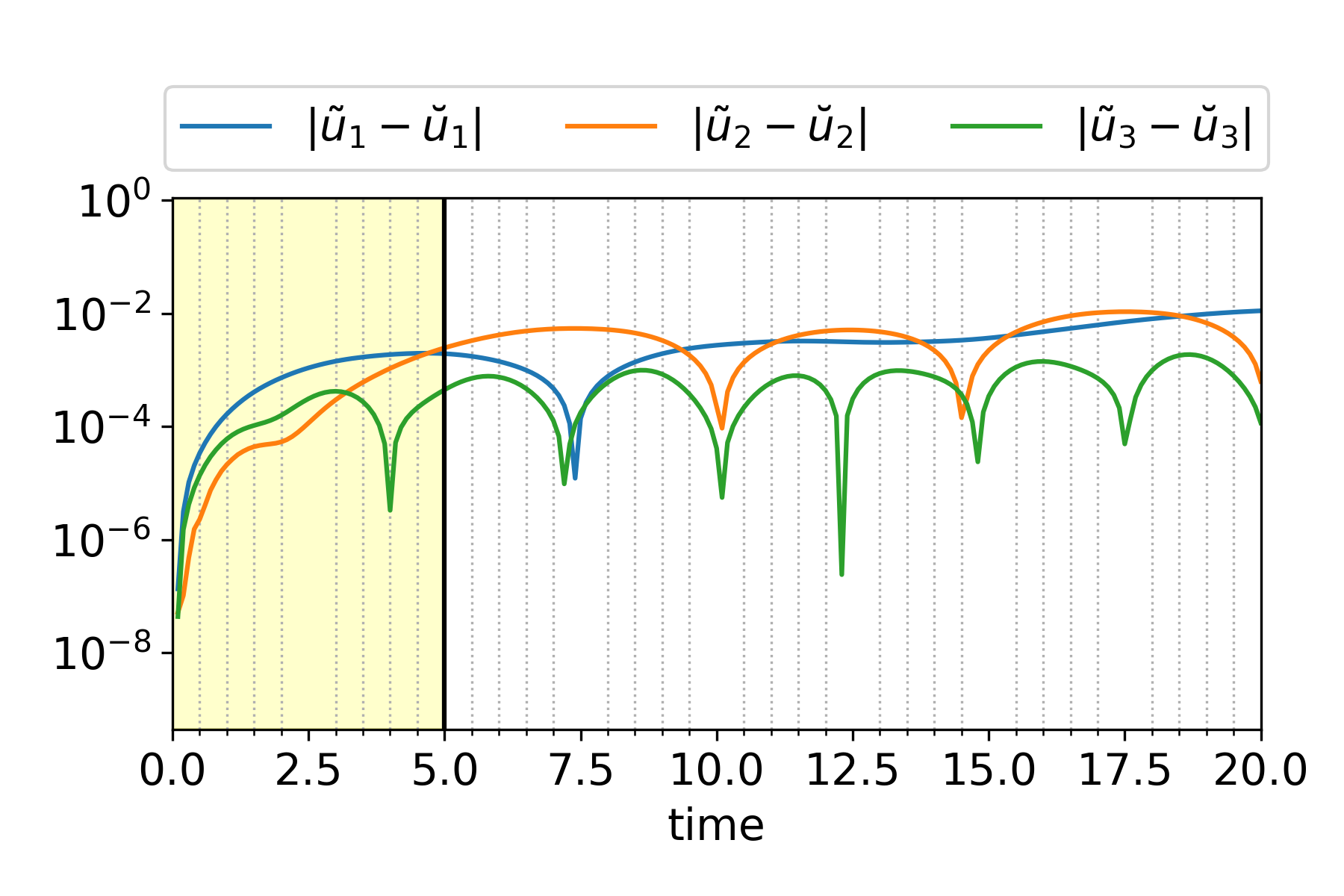}
	\caption{Absolute error}
	\end{subfigure}
	\caption{Linear wave equation: a comparison of the reduced coefficients obtained by the full and the reduced-order models. }
	\label{fig:wave-rcom}
\end{figure}

\begin{figure}[h]
	\centering
	\begin{subfigure}{0.328\textwidth}
		\includegraphics[width=1\linewidth]{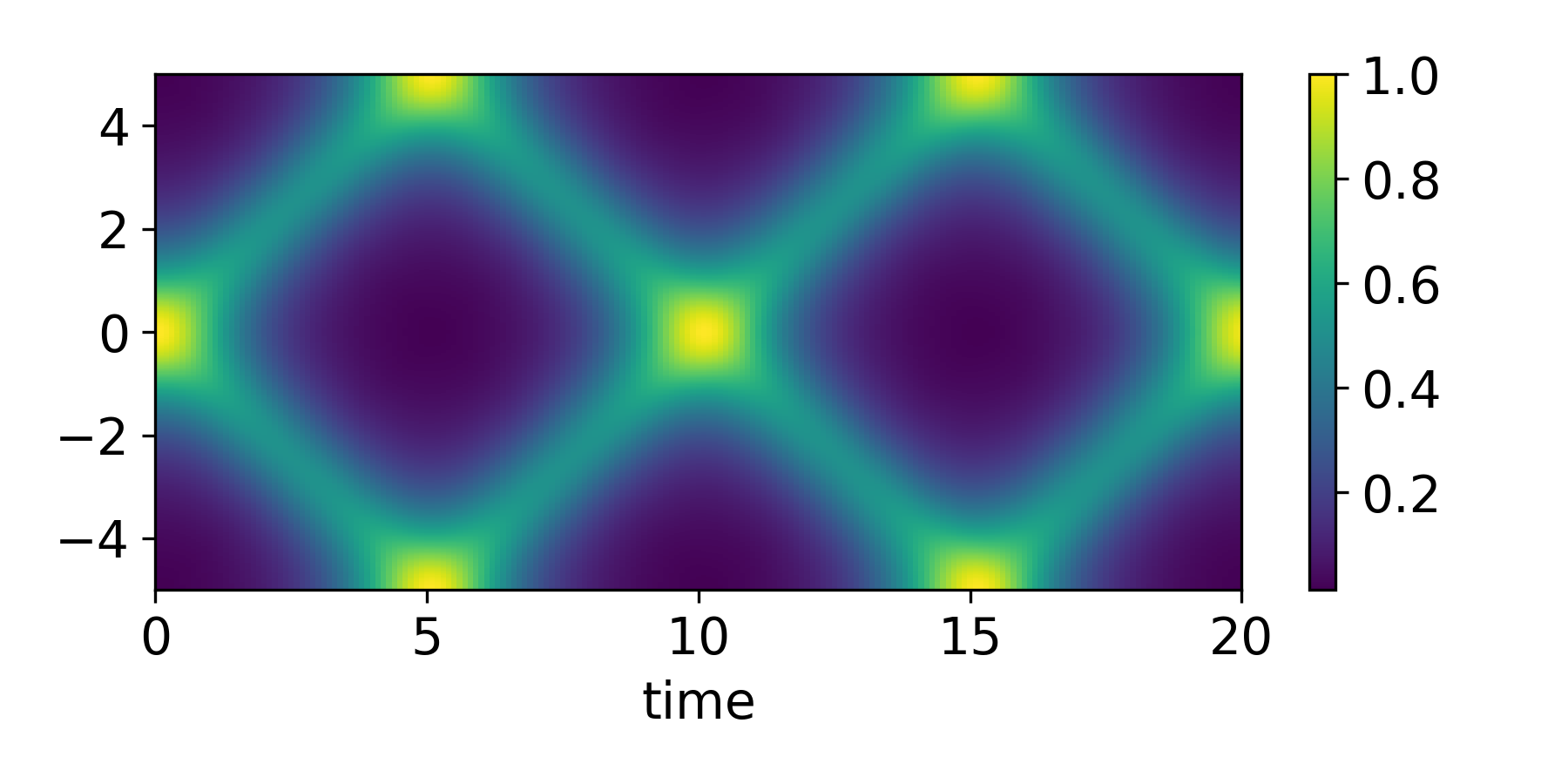}
		\caption{Ground truth}	
	\end{subfigure}
	\begin{subfigure}{0.328\textwidth}
		\includegraphics[width=1\linewidth]{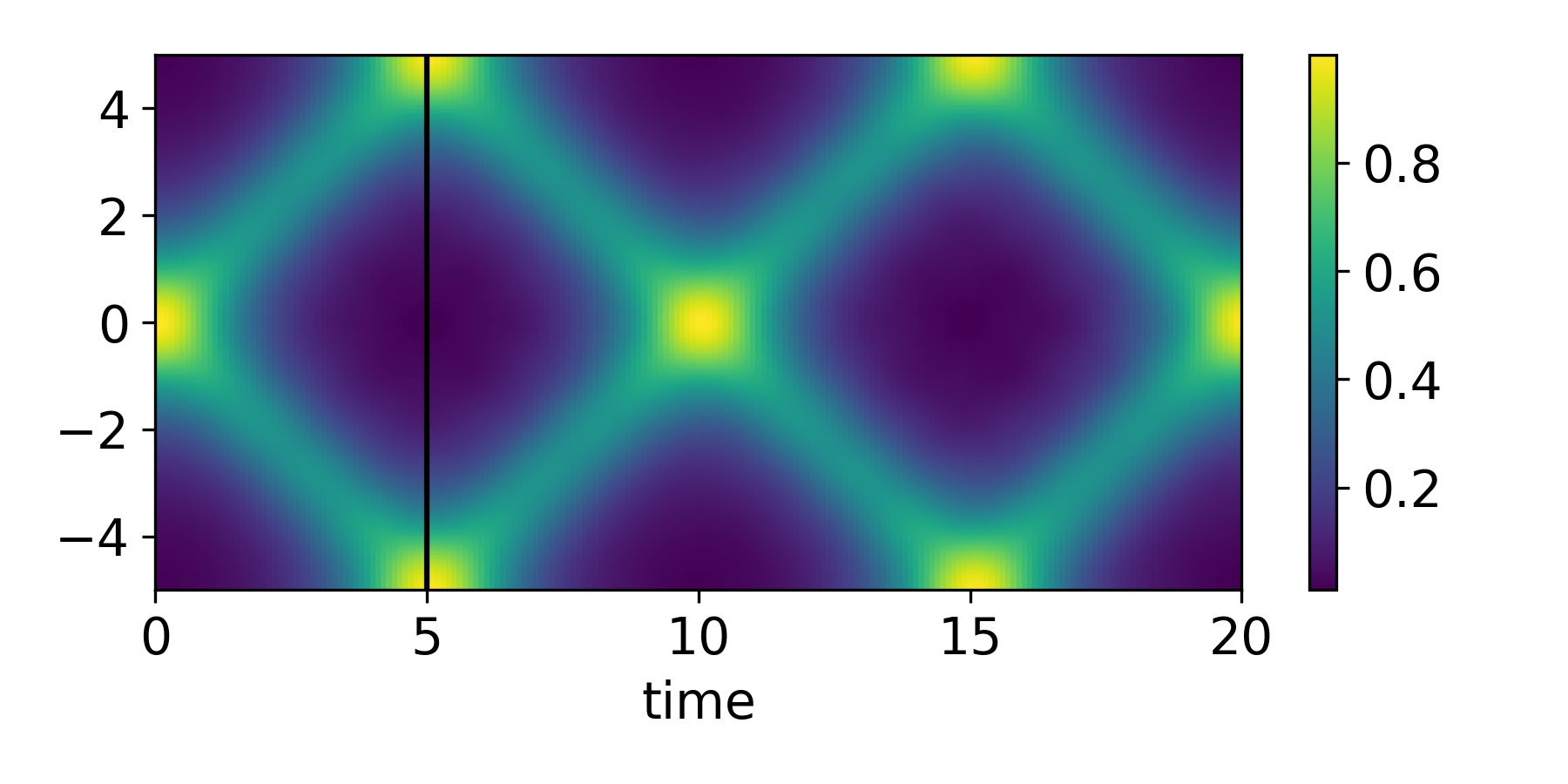} 
		\caption{OpInf}
	\end{subfigure}
	\begin{subfigure}{0.328\textwidth}
		\includegraphics[width=1\linewidth]{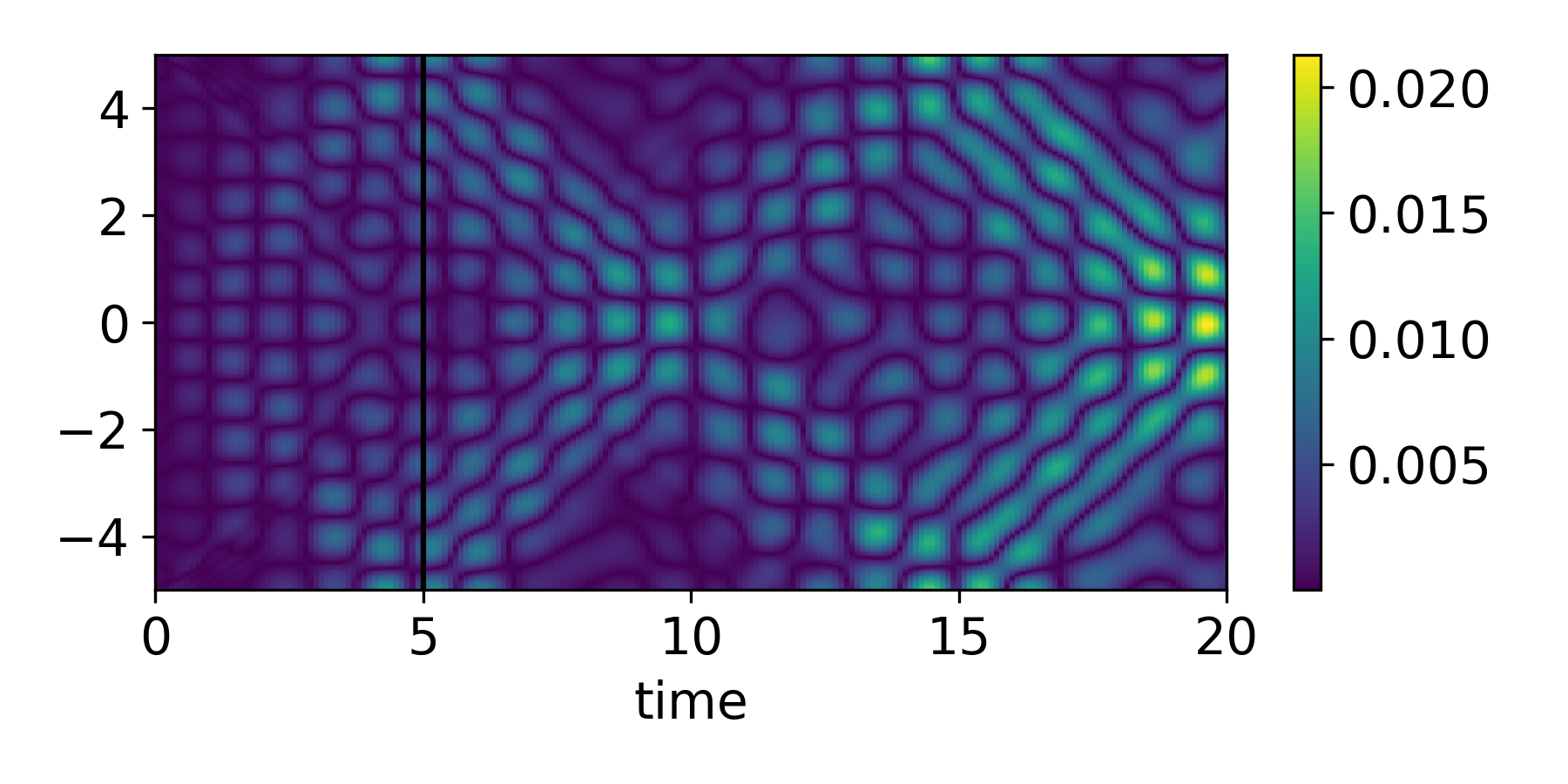}
		\caption{Absolute error}
	\end{subfigure}
	\caption{Linear wave equation: a comparison between the full and reduced-order models in full model dimension.  }
	\label{fig:wave-fcom}
\end{figure}

Finally, we compare the numerical solutions of the full model \eqref{eqn:wave_LIGEP} and the proposed OpInf model in the full dimension in \Cref{fig:wave-fcom}. For comparison, we plotted the absolute error between the full model and the OpInf model. Again, we have separate the training and test intervals with a vertical line to highlight the efficiency of the proposed method, which shows that the proposed methods capture the dynamics very well.

\subsection{Korteweg--de Vries equation}\label{subsec:kdv}
Our second test case is the one-dimensional Korteweg–de Vries (KdV) equation, which is commonly used in studies of shallow water waves, plasma physics, and internal waves. We consider the one-dimensional KdV equation of the form:
\begin{equation}\label{eqn:KdV}
	u_t + \eta u u_x+\gamma^2 u_{xxx}  = 0,
\end{equation}
where $\eta, \gamma \in \mathbb{R}$.

The multi-symplectic formalism can be revealed after defining the potential $\phi_x=u$, an auxiliary variable $w=\gamma v_x \phi_t+\frac{\gamma^2u^2}{2}$ and the momenta $v=\gamma u_x$, which have the following relations:
\begin{equation}\label{eqn:KdV-var}
	\begin{aligned}
		\frac{1}{2}u_t+w_x&=0, \quad&
		-\frac{1}{2}\phi_t-\gamma v_x&=-w+\frac{\eta}{2}u^2,\\
		\gamma u_x&=v,&
		-\phi_x&=-u.
	\end{aligned}
\end{equation}
Using the above equations, a multi-symplectic formulation \eqref{eqn:Ms_PDEs} for the KdV equation can be obtained with
\begin{equation*}
	K=\begin{bmatrix}
		0       &\frac{1}{2} & 0 & 0 \\
		-\frac{1}{2}       & 0 & 0 & 0 \\
		0      & 0 & 0 & 0 \\
		0      & 0 & 0 & 0  \\
	\end{bmatrix},\qquad
	L=
	\begin{bmatrix}
		0       &0 & 0 &1  \\
		0       & 0 & -\gamma & 0  \\
		0      & \gamma & 0 & 0  \\
		-1      & 0 & 0 & 0 \\
	\end{bmatrix},
\end{equation*}
$z = (\phi,u,v,w)^\top$, and the Hamiltonian $S(z):=\dfrac{v^2}{2}-uw+\dfrac{\eta u^3}{6}$.
The LIGEP method applied to the fully discrete KdV equation reads as \cite{yildiz23}
\begin{equation}\label{eqn:KdV-fom-elem}
	\begin{aligned}
		\frac{1}{2}\delta_t u_j^n+\delta_x^{1/2}\mu_tw_j^n &=0,&
		-\frac{1}{2}\delta_t\phi_j^n-\gamma \delta_x^{1/2}\mu_tv_j^n &=-\mu_t w_j^n+\frac{\eta}{2} u_j^n u_j^{n+1},\\
		\gamma \delta_x^{1/2}\mu_tu_j^n &=\mu_t v_j^n,&
		\delta_x^{1/2}\mu_t\phi_j^n &=\mu_t u_j^n.
	\end{aligned}
\end{equation}
After elimination of the auxiliary variables in the multi-symplectic form of the KdV equation \eqref{eqn:KdV-fom-elem}, the fully-discrete KdV equation reads as
\begin{equation}\label{eqn:KdV_LIGEP}
	\delta_tu_j^n+\frac{\eta}{2}\delta_x^{1/2}(u^n_ju^{n+1}_j)+\gamma^2\mu_t\left(\delta_x^{1/2}\right)^3u^n_j=0.
\end{equation}
The associated polarised discrete energy for the discrete KdV equation \eqref{eqn:KdV_LIGEP} is given as follows:
\begin{equation}\label{eqn:KdV-fom-pol-ener}
	\bar{\mathcal{E}}(t_n) = \, \frac {\Delta x}{6}\sum_{j=1}^{N}\left(- \gamma^2(\delta_x^{1/2}u^n_j)^2+2(\delta_x^{1/2}u^n_j)(\delta_x^{1/2} u^{n+1}_j)+\eta(u_j^n)^2 u_j^{n+1} \right).
\end{equation}
The POD basis for the multi-symplectic OpInf method is obtained by the following extended snapshot matrix:
$$\bl Z=\left[ \boldsymbol{\phi}(t_1),\ldots,\boldsymbol \phi(t_{N_t}),\bl u(t_1),\ldots,\bl u(t_{N_t}),\bl v(t_1),\ldots,\bl v(t_{N_t}),\bl w(t_1),\ldots,\bl w(t_{N_t})\right]\in \mathbb{R}^{N\times 4 N_t}.$$
We approximate the auxiliary variables for the above extended snapshot matrix as in \cite{yildiz23}.
Applying Kahan's method \cite{eidnes2020} to the reduced multi-symplectic model \eqref{MS-sd-rom} for time discretization, the reduced KdV equation can be written as follows:
\begin{equation}\label{eqn:KdV-LIGEP-ROM}
	\delta_t\blt u^n+\frac{\eta}{2}\tilde{D}_x V^\top (\blh u^n \circ\blh u^{n+1})+\gamma^2\mu_t\tilde{D}_x^3\blt u^n=0,
\end{equation}
where $\circ$ denotes the element-wise product.
The reduced KdV equation \eqref{eqn:KdV-LIGEP-ROM} conserves the following polarised energy:
\begin{equation}\label{eqn:KdV-rom-pol-ener}
	\begin{split}
		\bar{\mathcal{E}}_r(t_n) = & \, \frac {\Delta x}{6}\sum_{j=1}^{N}\Big(- \gamma^2( V \tilde{D}_x \blt u^n)_j^2+2( V\tilde{D}_x\blt u^n)_j( V \tilde{D}_x \blt u^{n+1})_j+ \eta (\blh u^n)_j^2 \blh u_j^{n+1} \Big).
	\end{split}
\end{equation}
Using the ROM for the KdV equation defined in \eqref{eqn:KdV-LIGEP-ROM}, we can recover the OpInf model from the following optimization problem:
\begin{equation}\label{eqn:KdV-OpInf-opt}
	\min_{\tilde{D}_x=-\tilde{D}_x^T}\left\|\delta_t\blt U^n+\frac{\eta}{2}\tilde{D}_x V^\top (\bl U^n \circ\bl U^n)+\gamma^2\mu_t\tilde{D}_x^3\blt U^n\right\|_F,
\end{equation}
where $\bl U^n=\left[ \bl u(t_1),\ldots,\bl u(t_{N_t})\right]\in \mathbb{R}^{N\times N_t}$ and $\blt{U}^n=V^T\bl{U} \in \mathbb{R}^{r\times N_t}$.
To test the proposed method, we set the constants to $\gamma=0.022$, $\nu=1$ and the spatial domain to $\Omega=[0,P]$ with $P=2$ as in \cite{eidnes2020}. Moreover, we set the spatial step-size to $\Delta x=0.004$ and temporal step-size to $\Delta t=0.1$. The boundary condition is set to periodic and the initial condition to  $u(x,0)=0.4\sech( x-P/2))$.
To construct the reduced coefficient matrix $\blt{U}$, we simulated the full model \eqref{eqn:KdV_LIGEP} up to the final time $T=50$. We used the reduced coefficients up to time $T=15$ to construct the coefficient matrix for training. We have set the reduced dimension to $r=16$ to construct the OpInf model.

In \Cref{fig:kdv-OpInf-ener} we have plotted the polarised discrete energy \eqref{eqn:KdV-fom-pol-ener} preserved by the full model \eqref{eqn:KdV_LIGEP} and the relative error \eqref{eqn:rel_ener_err} between the full and reduced polarised discrete energy \eqref{eqn:KdV-rom-pol-ener}. The figure shows that both, the full and the OpInf model preserve the polarised energy.

\begin{figure}[h]
	\centering
	\begin{subfigure}{0.328\textwidth}
	\includegraphics[width=1\linewidth]{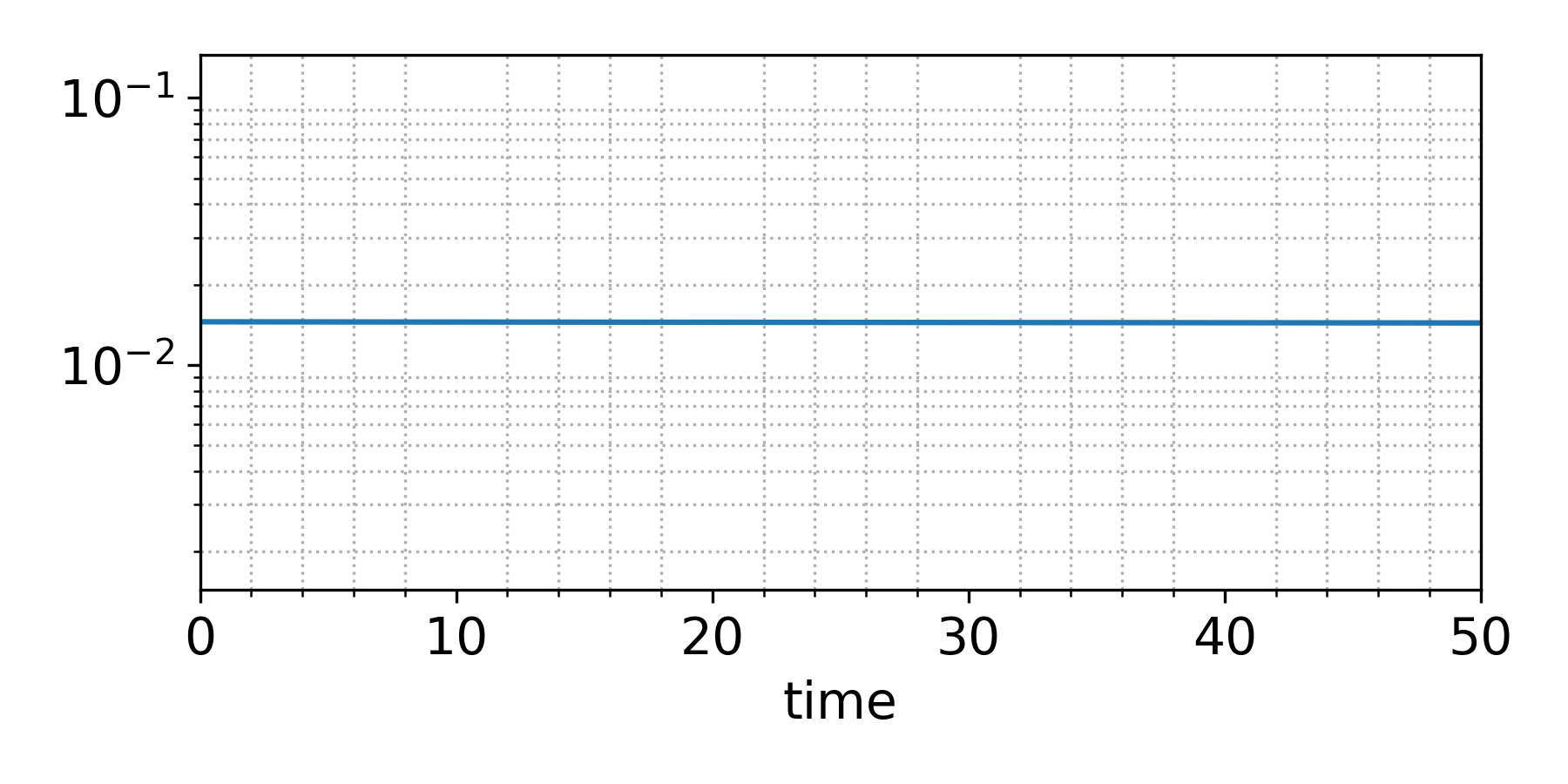}
	\caption{Ground truth energy}	
	\end{subfigure}
	\begin{subfigure}{0.328\textwidth}
	\includegraphics[width=1\linewidth]{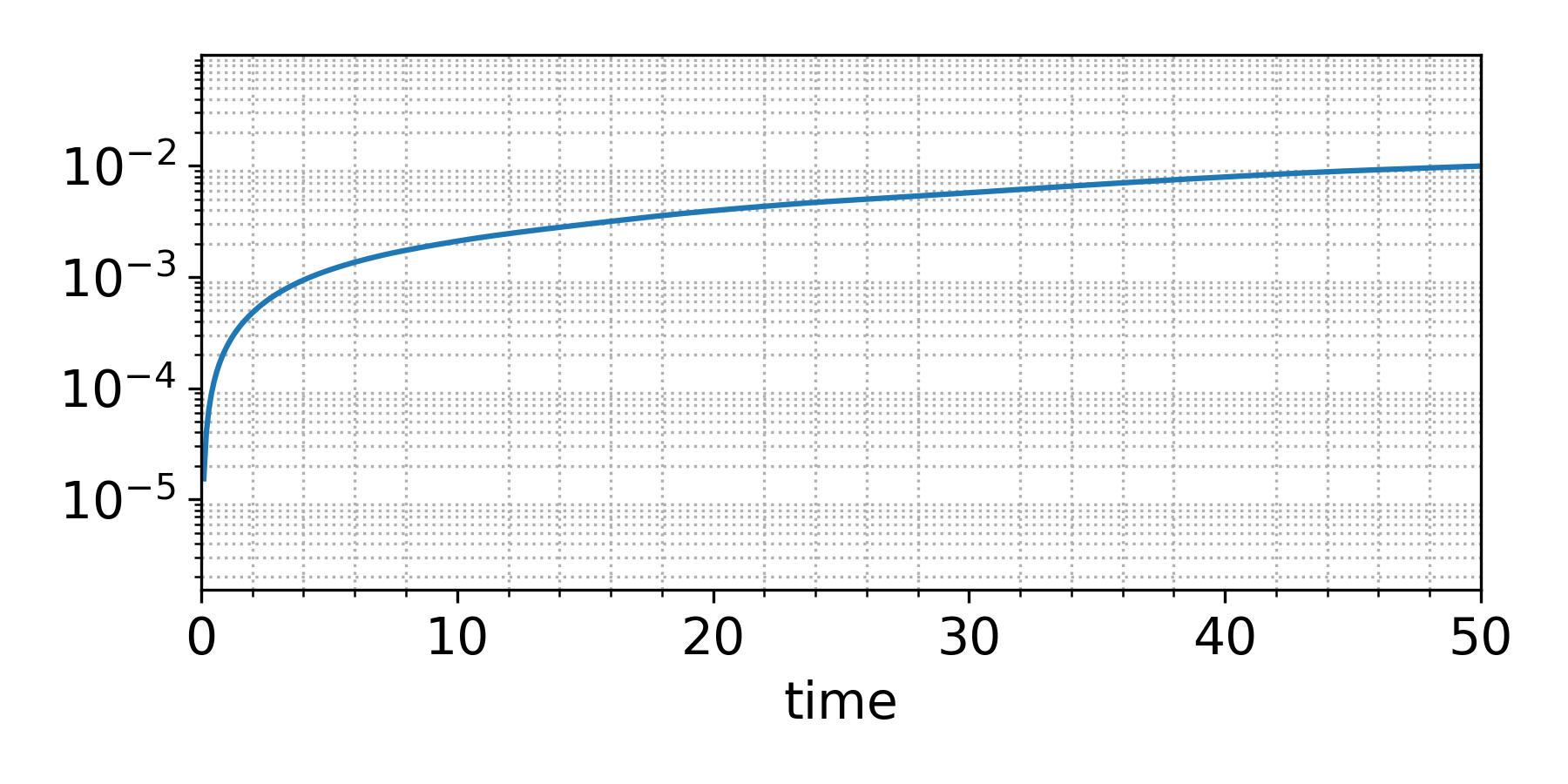} 
	\caption{Relative error}
	\end{subfigure}
	\caption{KdV: discrete energy comparison between full and reduced-order models.}
	\label{fig:kdv-OpInf-ener}
\end{figure}

Again, we compare the first three reduced coefficients $\tilde{u}_i(t)$ of the ground truth model with the coefficients obtained by the proposed OpInf method $\breve{u}_i(t)$ in \Cref{fig:kdv-rcom}. As in the previous example, we have plotted a vertical line to separate the training and testing interval to test the generality of the proposed model. The absolute errors shown in \Cref{fig:kdv-rcom} indicate that the absolute errors do not grow rapidly outside the training interval, indicating that the proposed OpInf model captures the reduced dynamics well. 

\begin{figure}[h]
	\centering
	\begin{subfigure}{0.328\textwidth}
		\includegraphics[width=1\linewidth]{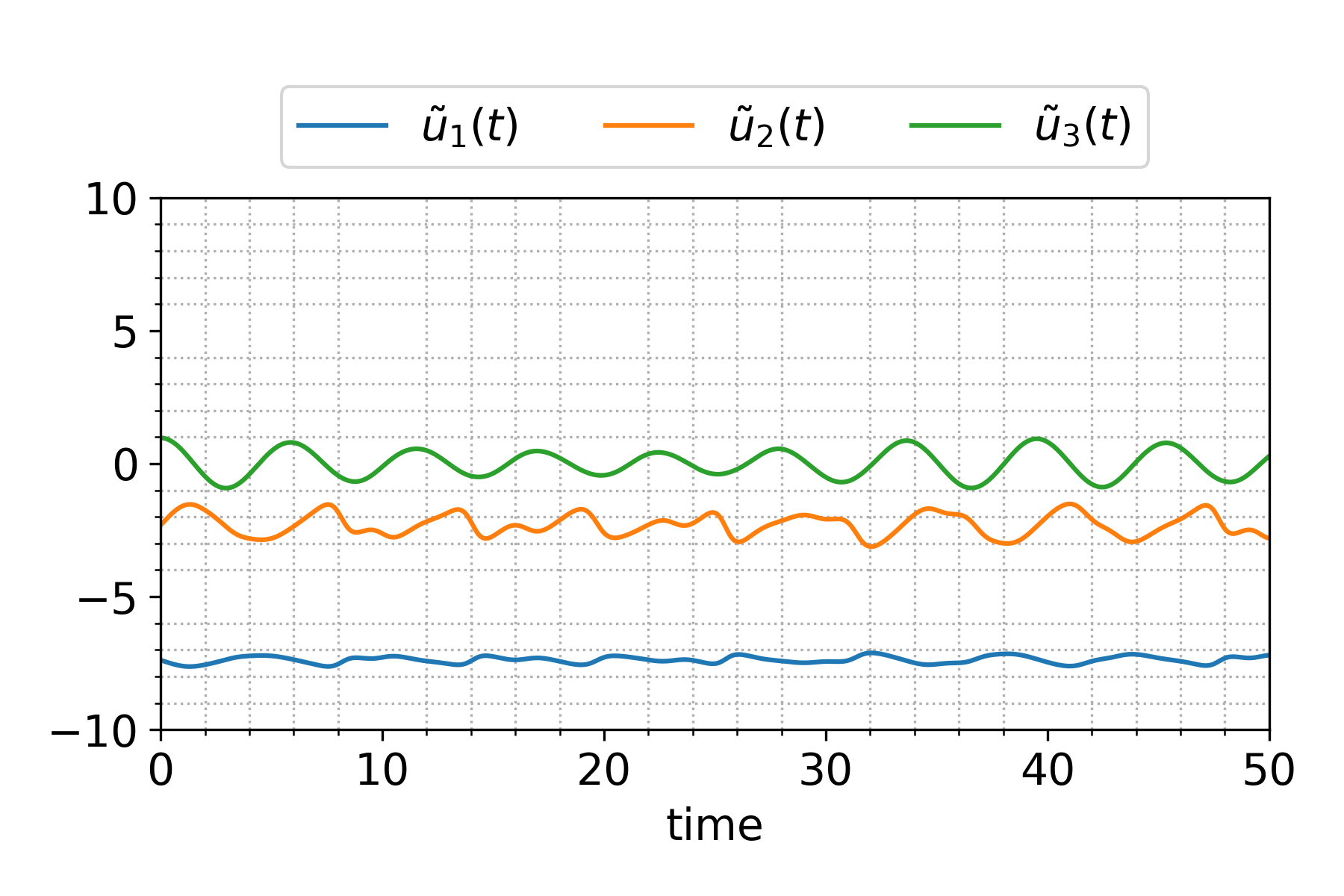}
		\caption{Ground truth}	
	\end{subfigure}
	\begin{subfigure}{0.328\textwidth}
		\includegraphics[width=1\linewidth]{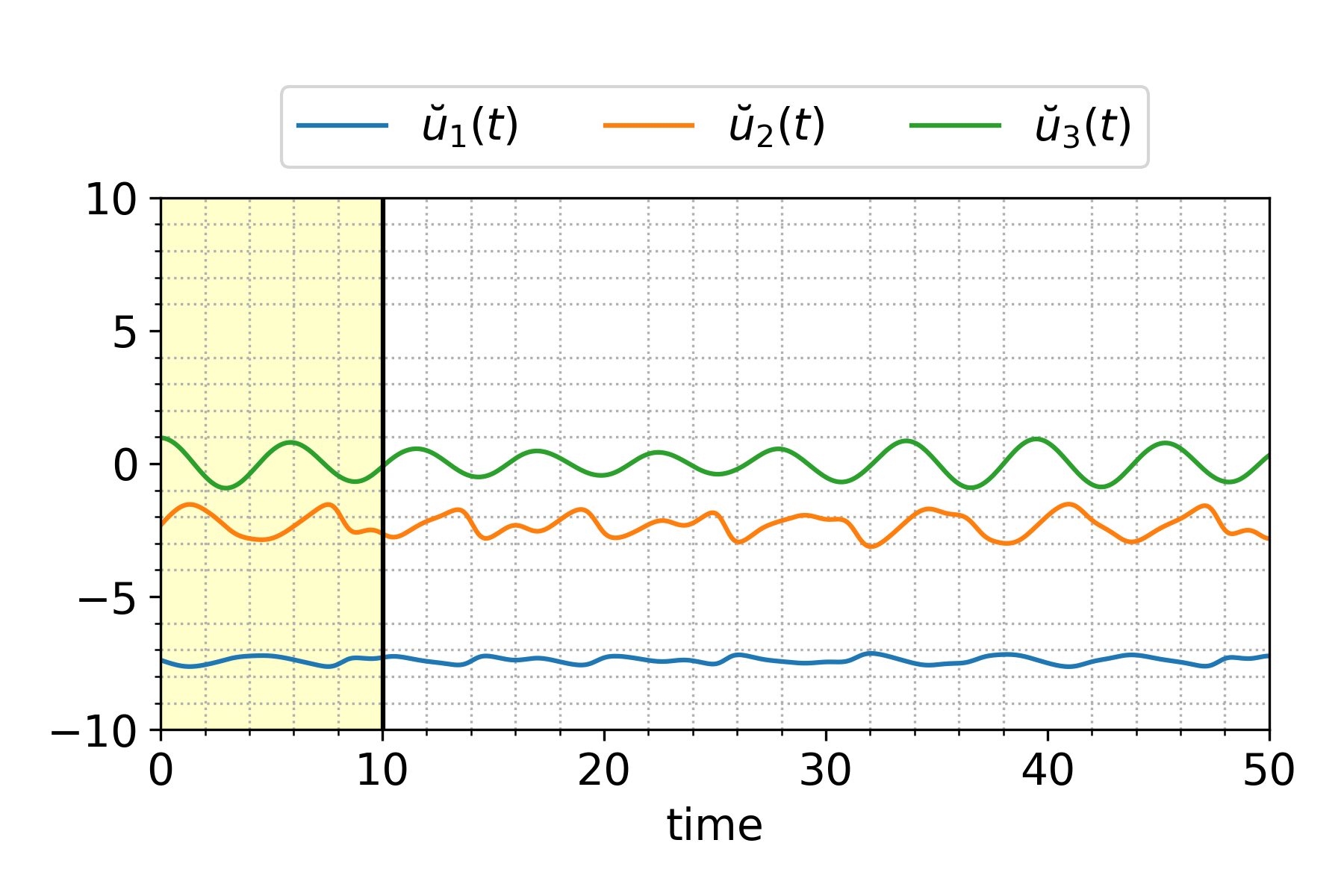} 
		\caption{OpInf}
	\end{subfigure}
	\begin{subfigure}{0.328\textwidth}
		\includegraphics[width=1\linewidth]{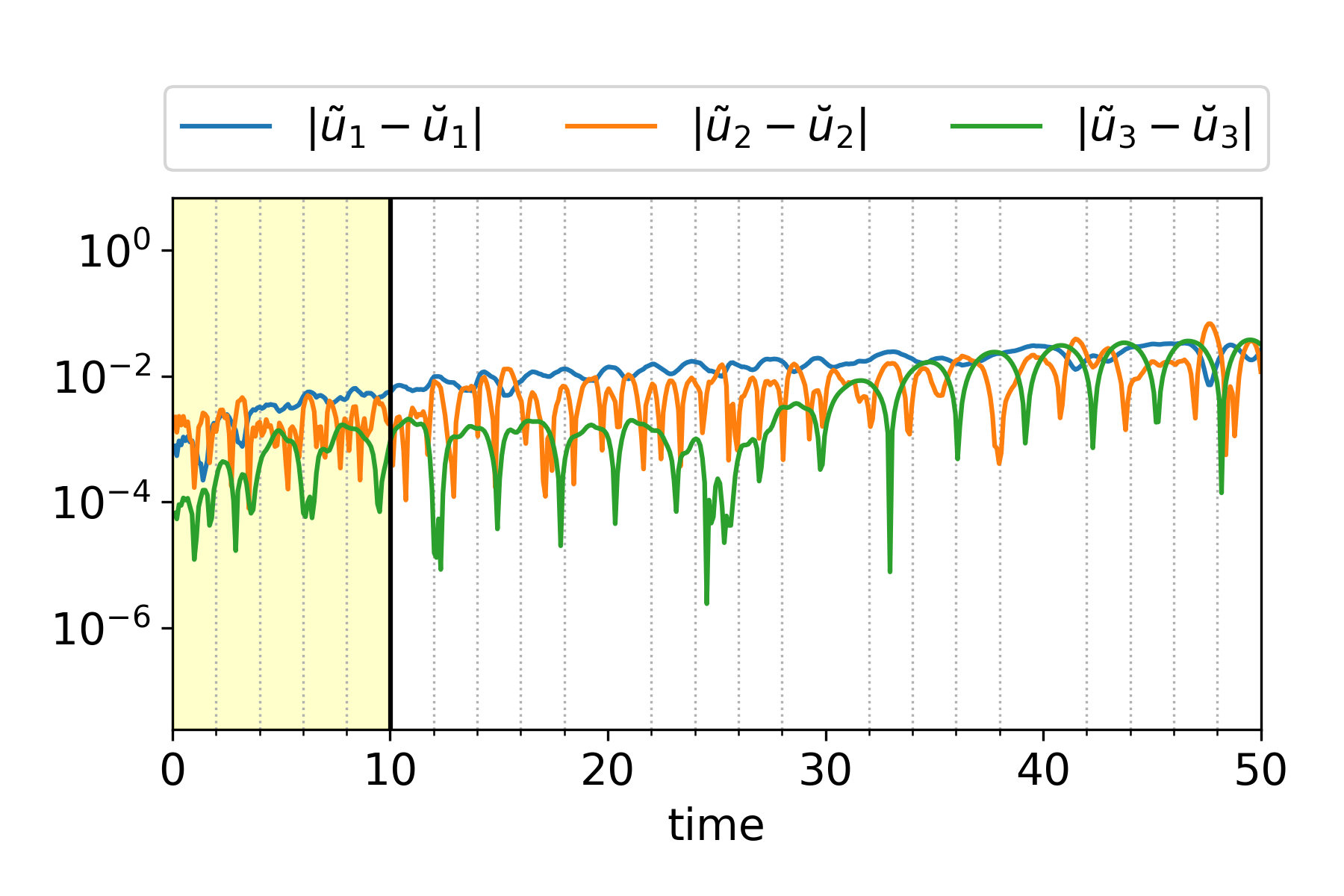}
		\caption{Absolute error}
	\end{subfigure}
	\caption{KdV:  a comparison of the reduced coefficients obtained by the full and the reduced-order models.  }
	\label{fig:kdv-rcom}
\end{figure}

We also compare the accuracy of the proposed OpInf model in the full dimension in \Cref{fig:kdv-fcom}. The figure reveals that the absolute error between the full model \eqref{eqn:KdV_LIGEP} and the learned OpInf is in good agreement, as in the previous example. 
 
\begin{figure}[h]
	\centering
	\begin{subfigure}{0.328\textwidth}
		\includegraphics[width=1\linewidth]{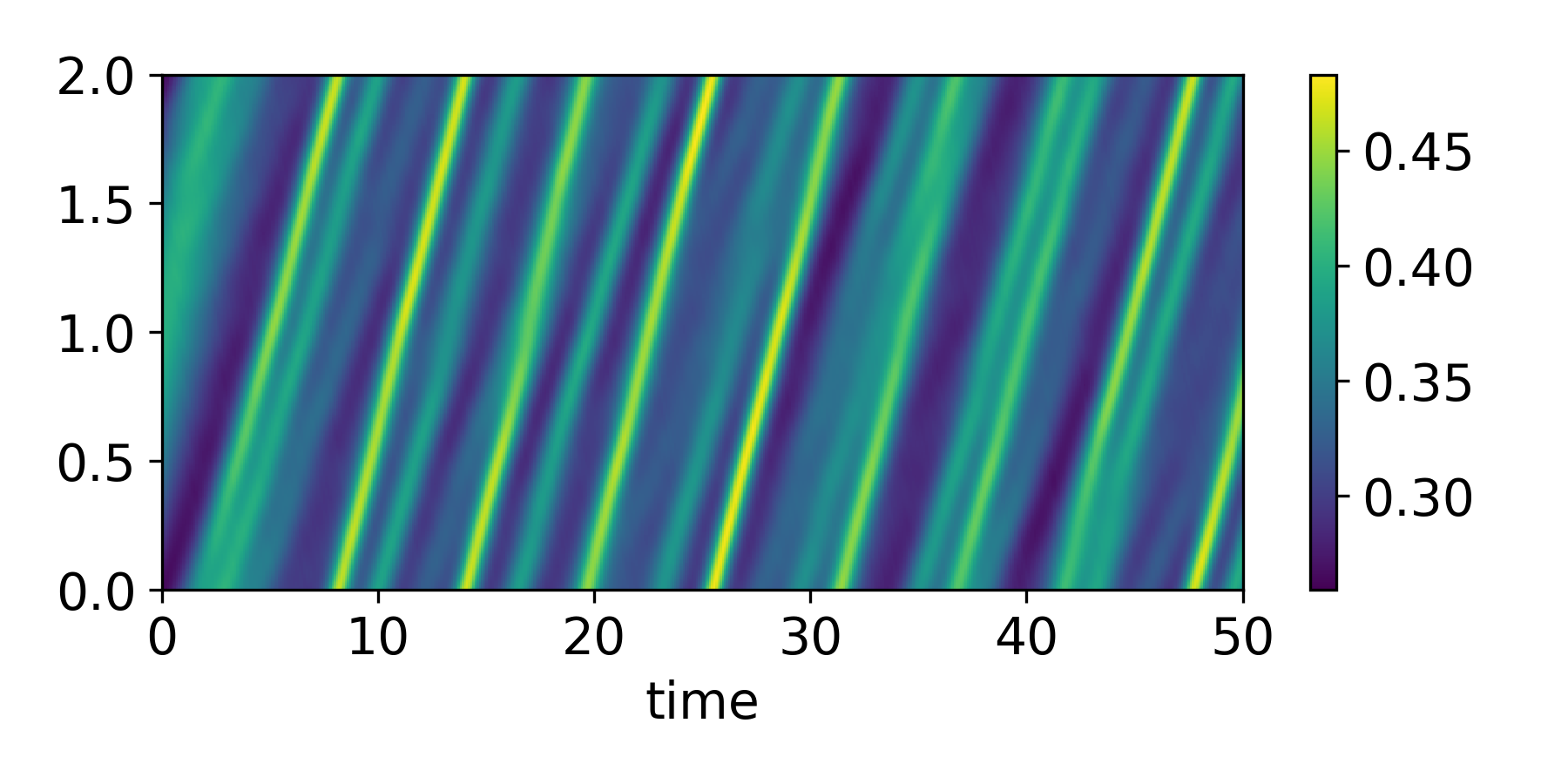}
		\caption{Ground truth}	
	\end{subfigure}
	\begin{subfigure}{0.328\textwidth}
		\includegraphics[width=1\linewidth]{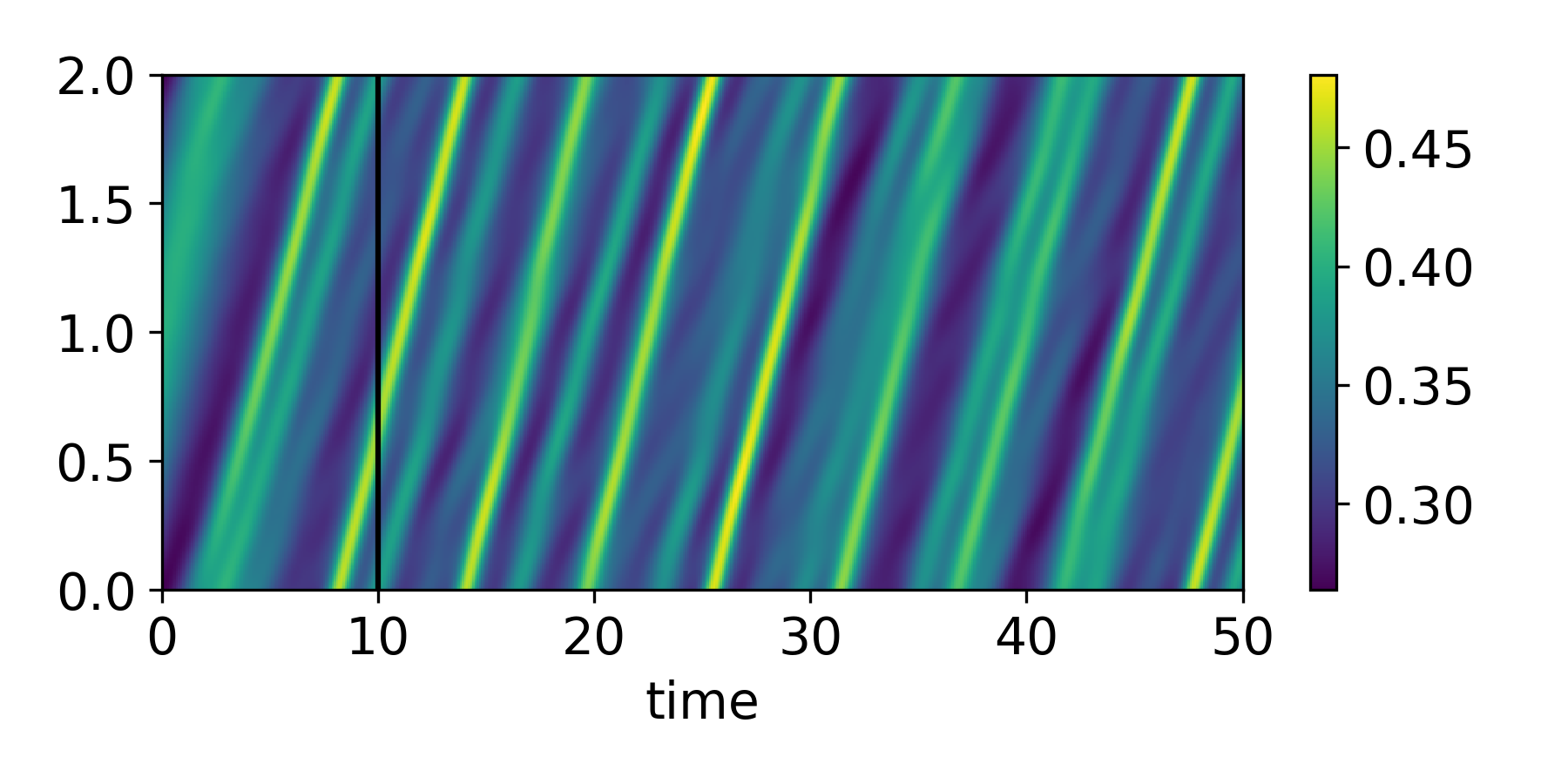} 
		\caption{OpInf}
	\end{subfigure}
	\begin{subfigure}{0.328\textwidth}
		\includegraphics[width=1\linewidth]{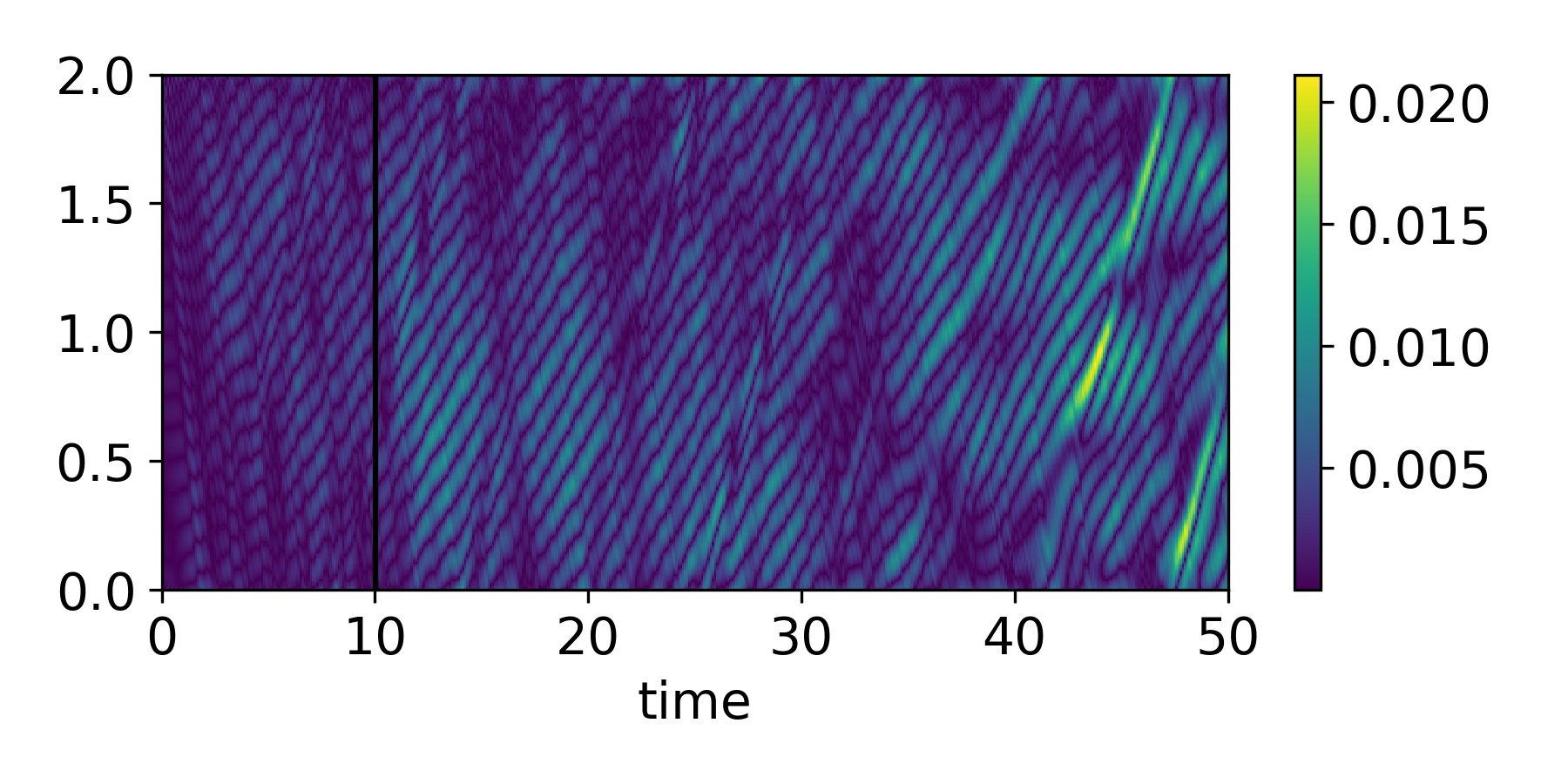}
		\caption{Absolute error}
	\end{subfigure}
	\caption{KdV:   a comparison between the full and reduced-order models in full model dimension.  }
	\label{fig:kdv-fcom}
\end{figure}

\subsection{Zakharov-Kuznetsov equation}\label{subsec:zk}
In our last example, we test the performance of the proposed model on higher-dimensional multi-symplectic PDEs. To do this, we test the performance of the proposed model with the Zakharov-Kuznetsov (ZK) equation. The ZK equation is a non-integrable PDE \cite{nishiyama2012conservative} and can be considered as generalisation of the KdV equation.
The ZK equation introduced in \cite{zakharov1974three}, which has following form 
\begin{equation}
	u_t + u u_x + u_{xxx} + u_{x y y}=0.
	\label{eqn:Zakharov-Kuznetsov}
\end{equation}
The ZK equation can be equivalently represented as a system of first-order PDEs \cite{bridges01} as follows:
\begin{equation}\label{eq:ZKsystem}
	\begin{split}
		\phi_x &= u,\\
		\frac{1}{2}\phi_t + v_x + w_y &= p - \frac{1}{2}u^2,\\
		w_x - v_y &= 0,\\
		-\frac{1}{2}u_t - p_x &= 0,\\
		-u_x + q_y &= -v,\\
		-q_x - u_y &= -w,
	\end{split}
\end{equation}
which reveals the following multi-symplectic formulation
\begin{equation}\label{eqn:ZK_2d_ms}
	K z_t + L^1 z_x + L^2 z_y = \nabla S(z), \quad z \in \mathbb{R}^6, \quad (x,y,t)\in \mathbb{R}^2 \times \mathbb{R},
\end{equation}
with $z = (p,u,q,\phi,v,w)^T$, the Hamiltonian $S(z) = u p - \frac{1}{2}(v^2+w^2) - \frac{1}{6}u^3$, and skew-symmetric matrices defined as follows: 
\begin{equation*}
K=\begin{bmatrix}
     0 & 0 & 0 & 0 & 0 & 0 \\
     0 & 0 & 0 & \frac{1}{2} & 0 & 0 \\
     0 & 0 & 0 & 0 & 0 & 0 \\
     0 & -\frac{1}{2} & 0 & 0 & 0 & 0 \\
     0 & 0 & 0 & 0 & 0 & 0 \\
     0 & 0 & 0 & 0 & 0 & 0
\end{bmatrix},\quad
L^1=\begin{bmatrix}
     0 & 0 & 0 & 1 & 0 & 0 \\
     0 & 0 & 0 & 0 & 1 & 0 \\
     0 & 0 & 0 & 0 & 0 & 1 \\
    -1 & 0 & 0 & 0 & 0 & 0 \\
     0 & -1 & 0 & 0 & 0 & 0 \\
     0 & 0 & -1 & 0 & 0 & 0
\end{bmatrix},\quad
L^2=\begin{bmatrix}
    0 & 0 & 0 & 0 & 0 & 0 \\
    0 & 0 & 0 & 0 & 0 & 1 \\
    0 & 0 & 0 & 0 & -1 & 0 \\
    0 & 0 & 0 & 0 & 0 & 0 \\
    0 & 0 & 1 & 0 & 0 & 0 \\
    0 & -1 & 0 & 0 & 0 & 0
\end{bmatrix}.
\end{equation*}
Following \cite{eidnes2020}, we discretize the multi-symplectic form of the ZK equation \eqref{eqn:ZK_2d_ms} with the linearly implicit global energy preserving method, which yields the following equation
\begin{equation}
	\delta_t u_{j,k}^n + \frac{1}{2} (D_x(u^n u^{n+1}))_{j,k} + \mu_t(D_x^3(u^n))_{j,k} + \mu_t(D_xD_y^2(u^n))_{j,k} =0.
	\label{eq:ZK_LIGEP}
\end{equation}
The above equation is preserving the following discrete polarized energy
\begin{equation}\label{eqn:zk-pol-ener-fom}
		\begin{split}
			\bar{\mathcal{E}}(t_n) =& \frac{1}{6} \Delta x \, \Delta y \sum_{j=1}^{N} \sum_{k=1}^{N} \Big( 2 (D_xu^n)_{j,k} (D_xu^{n+1})_{j,k} + ((D_xu^n)_{j,k})^2 \\
			&+ 2 (D_yu^n)_{j,k} (D_yu^{n+1})_{j,k} + ((D_yu^n)_{j,k})^2 - (u_{j,k}^n)^2 u_{j,k}^{n+1} \Big).
		\end{split}
\end{equation}

In this example, we use partial knowledge of the auxiliary variables to construct the extended snapshot matrix as in \cite{uzunca2023global}. Using the equation $\phi_x = u$, we construct the following extended snapshot matrix:
$$\bl Z=\left[ \bl u(t_1),\ldots,\bl u(t_{N_t}),\boldsymbol{\phi}(t_1),\ldots,\boldsymbol \phi(t_{N_t})\right]\in \mathbb{R}^{N\times 2 N_t}.$$
After obtaining the POD basis, we can write the ROM obtained with LIGEP method \cite{yildiz23} as follows:
\begin{equation}\label{eqn:ZK-LIGEP-ROM}
	\delta_t\blt u^n+\frac{1}{2}\tilde{D}_x V^\top (\blh u^n \circ\blh u^{n+1})+ \mu_t(\tilde{D}_x^3+\tilde{D}_x\tilde{D}_y^2)\blt u^n=0.
\end{equation}
The corresponding global energy preserved by the ROM \eqref{eqn:ZK-LIGEP-ROM} reads as
\begin{equation}\label{eqn:zk-pol-ener-rom}
	\begin{split}
		\bar{\mathcal{E}}_r(t_n) =& \frac{1}{6} \Delta x \, \Delta y \sum_{j=1}^{N} \sum_{j=1}^{N} \Big( 2 (V \tilde{D}_x\blt{u}^n)_{j,k} (V\tilde{D}_x\blt{u}^{n+1})_{j,k} + ((V\tilde{D}_x\blt{u}^n)_{j,k})^2 \\
		&+ 2 (V\tilde{D}_y\blt{u}^n)_{j,k} (V\tilde{D}_y\blt{u}^{n+1})_{j,k} + ((V\tilde{D}_y\blt{u}^n)_{j,k})^2 - (\blh{u}_{j,k}^n)^2 \blh{u}_{j,k}^{n+1} \Big).
	\end{split}
\end{equation}
The optimization problem to recover the reduced operators via OpInf can be written as
\begin{equation}\label{eqn:ZK-OpInf-opt}
	\min_{\tilde{D}_x=-\tilde{D}_x^T, \tilde{D}_y=-\tilde{D}_y^T}\left\|\delta_t\blt U^n+\frac{1}{2}\tilde{D}_x V^\top (\bl U^n \circ\bl U^n)+ \mu_t(\tilde{D}_x^3+\tilde{D}_x\tilde{D}_y^2)\blt U^n\right\|_F,
\end{equation}
where $\bl U^n=\left[ \bl u(t_1),\ldots,\bl u(t_{N_t})\right]\in \mathbb{R}^{N\times N_t}$ and $\blt{U}^n=V^T\bl{U} \in \mathbb{R}^{r\times N_t}$.

To test our proposed method, we consider the collision of double line solitons with the following initial condition \cite{chen2011multi}:
\begin{equation*}
	u(x,y,0)=\sum_{j=1}^{2}3c_j\sech^2\left(0.5\sqrt{\dfrac{c_j}{\varepsilon}}\left((x-x_j)\cos\theta+(y-y_j)\sin\theta\right) \right),
\end{equation*}
where $\varepsilon=0.01$, $\theta=0$, $y_1=y_2=0$, $c_1=0.45$, $c_2=0.25$, $x_1=2.5$ and $x_2=3.3$. To obtain the discrete differential operators $D_x$ and $D_y$ we used the central difference method. We consider a square spatial domain $\Omega=[0,8]\times[0,8]$ with periodic boundary conditions. We discretized the spatial domain by dividing each spatial direction into $N=50$ sub-intervals, resulting in spatial step sizes $\Delta x=\Delta y=0.16$ and full model dimension $N^2=2500$. The full model \eqref{eq:ZK_LIGEP} was simulated up to the final time $T=50$ with the temporal step-size $\Delta t=0.025$. We have constructed the reduced coefficients matrix $\blt U$ with the snapshots obtained up to time $T=25$. We have chosen to reduce the dimension to $r=64$ to test the proposed method.

In \Cref{fig:zk-OpInf-ener} we show the discrete polarised energy \eqref{eqn:zk-pol-ener-fom} obtained from the full model \eqref{eq:ZK_LIGEP} and the relative energy error between the full and reduced-order models, which shows that the model resulting from the OpInf captures the global energy qualitatively well. 

As in the previous examples, we compare our method with the ground truth model obtained by \eqref{eq:ZK_LIGEP} in full and reduced dimensions in \Cref{fig:zk-fcom} and \Cref{fig:zk-rcom}, respectively. In \Cref{fig:zk-rcom} we have used a vertical line to highlight the training and test intervals as in the previous examples. \Cref{fig:zk-fcom,fig:zk-rcom} shows that the proposed ROM is also suitable for higher dimensional problems. 

\begin{figure}[h]
	\centering
	\begin{subfigure}{0.328\textwidth}
		\includegraphics[width=1\linewidth]{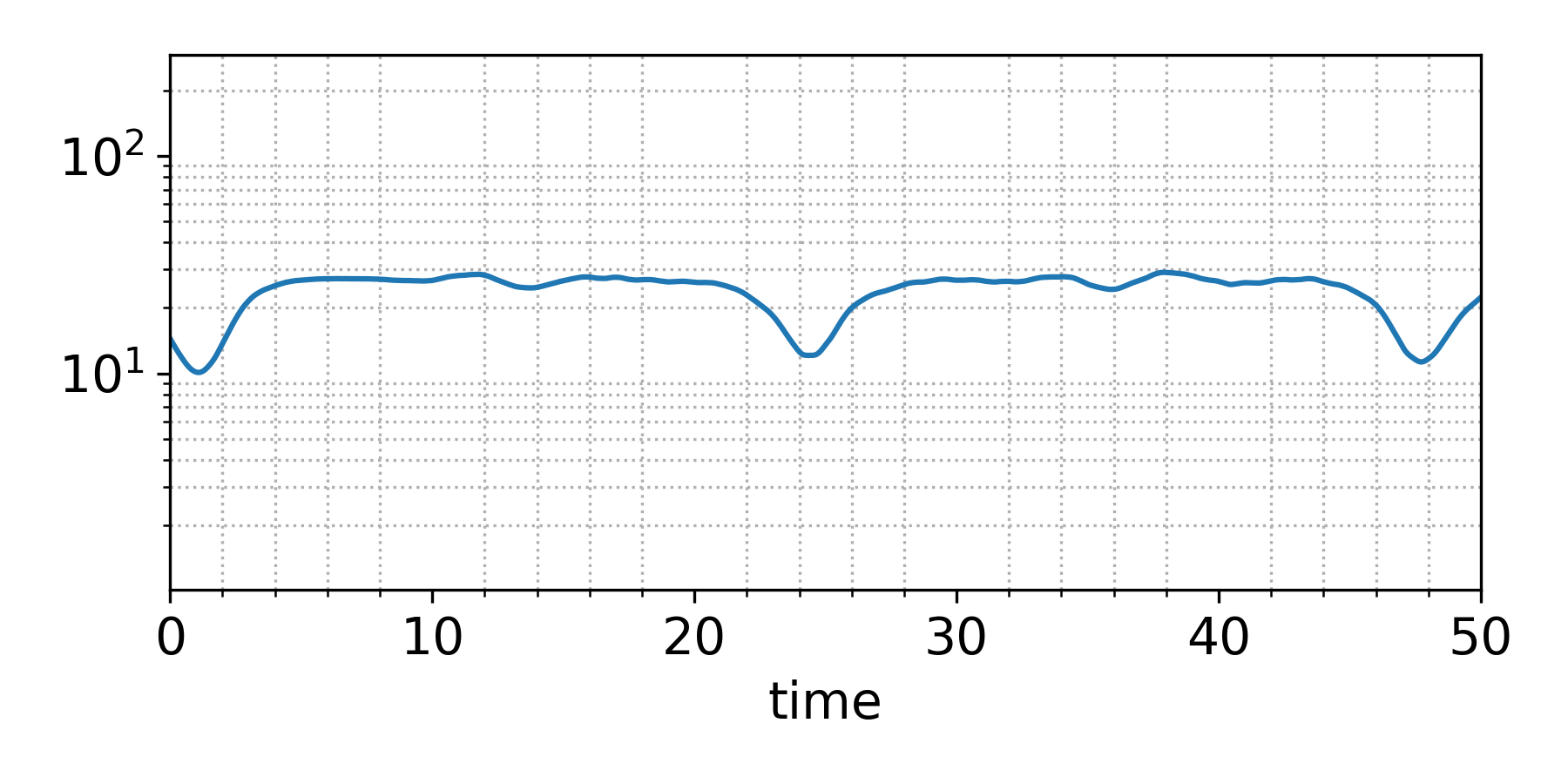}
		\caption{Ground truth energy}	
	\end{subfigure}
	\begin{subfigure}{0.328\textwidth}
		\includegraphics[width=1\linewidth]{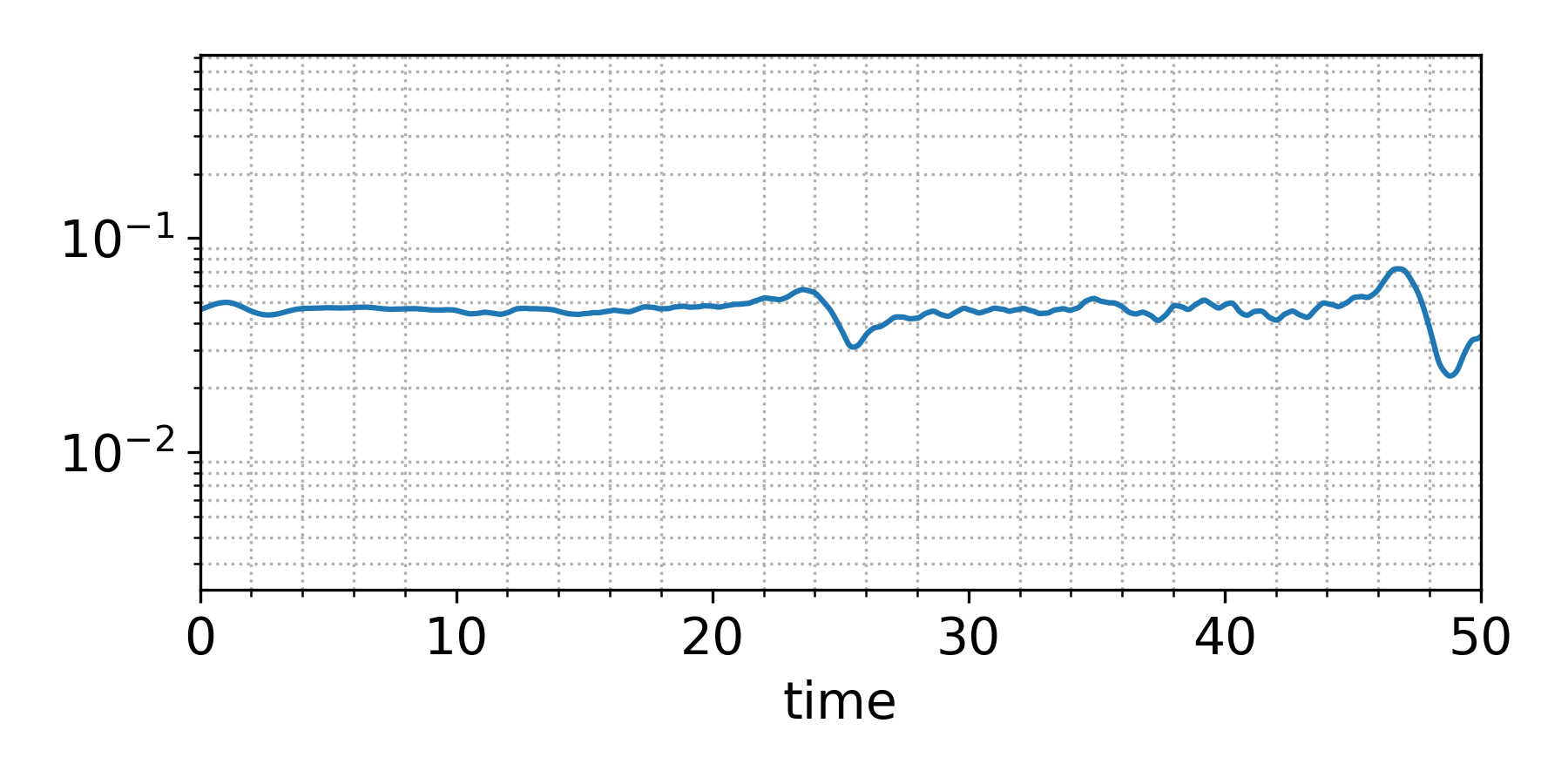} 
		\caption{Relative error}
	\end{subfigure}
	\caption{ZK: discrete energy comparison between full and reduced-order models.}
	\label{fig:zk-OpInf-ener}
\end{figure}

\begin{figure}[h]
	\centering
	\begin{subfigure}{0.328\textwidth}
		\includegraphics[width=1\linewidth]{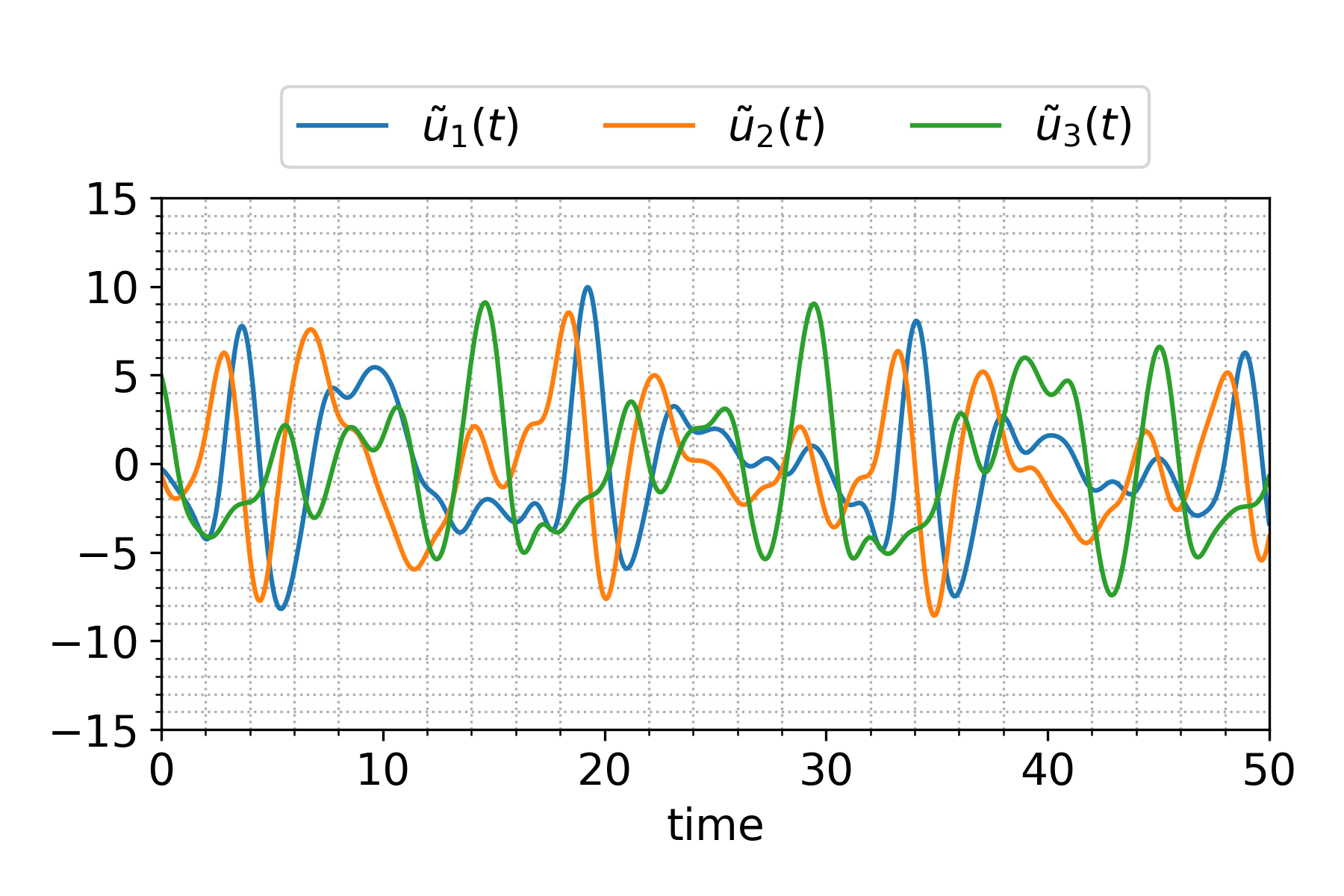}
		\caption{Ground truth}	
	\end{subfigure}
	\begin{subfigure}{0.328\textwidth}
		\includegraphics[width=1\linewidth]{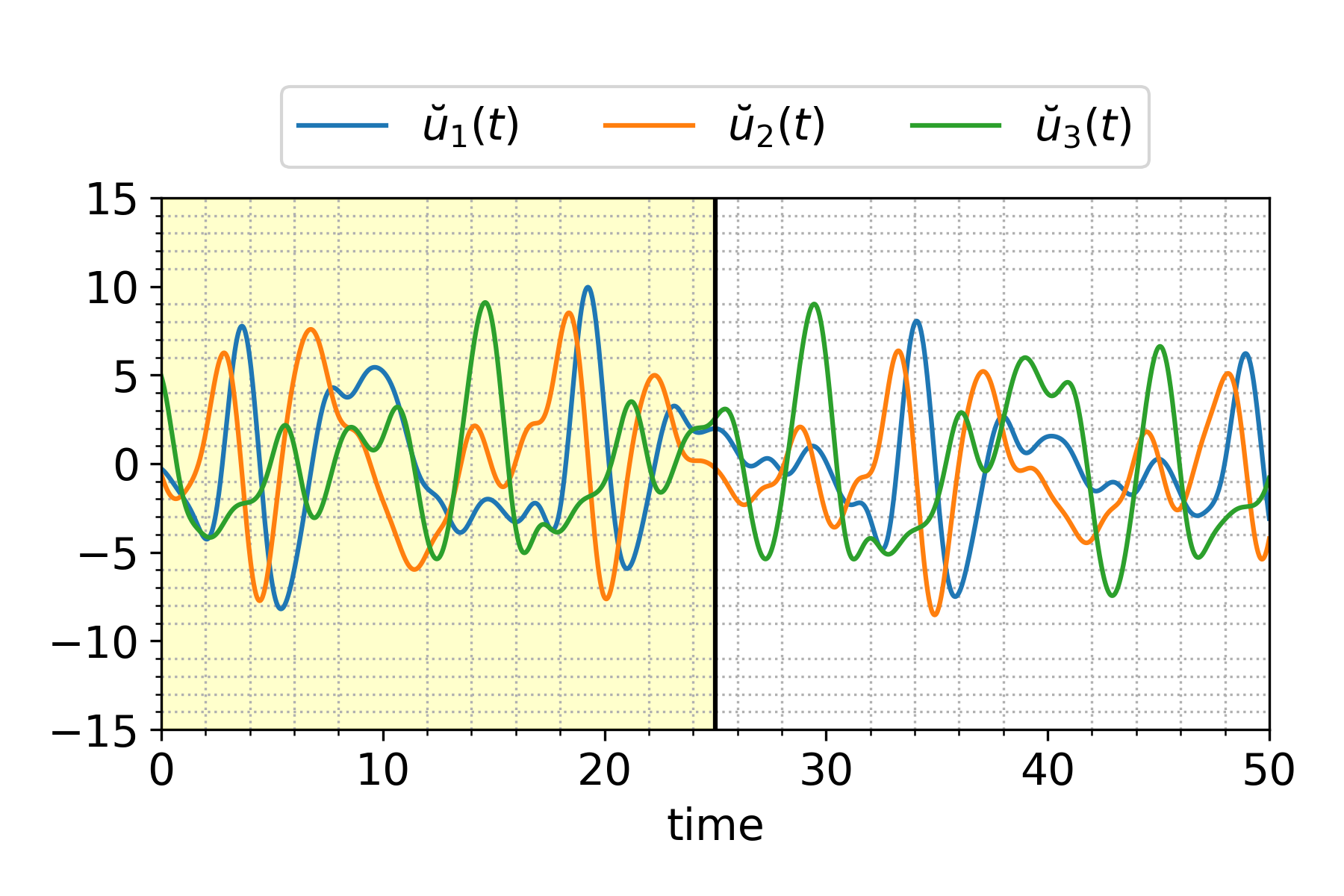} 
		\caption{OpInf}
	\end{subfigure}
	\begin{subfigure}{0.328\textwidth}
		\includegraphics[width=1\linewidth]{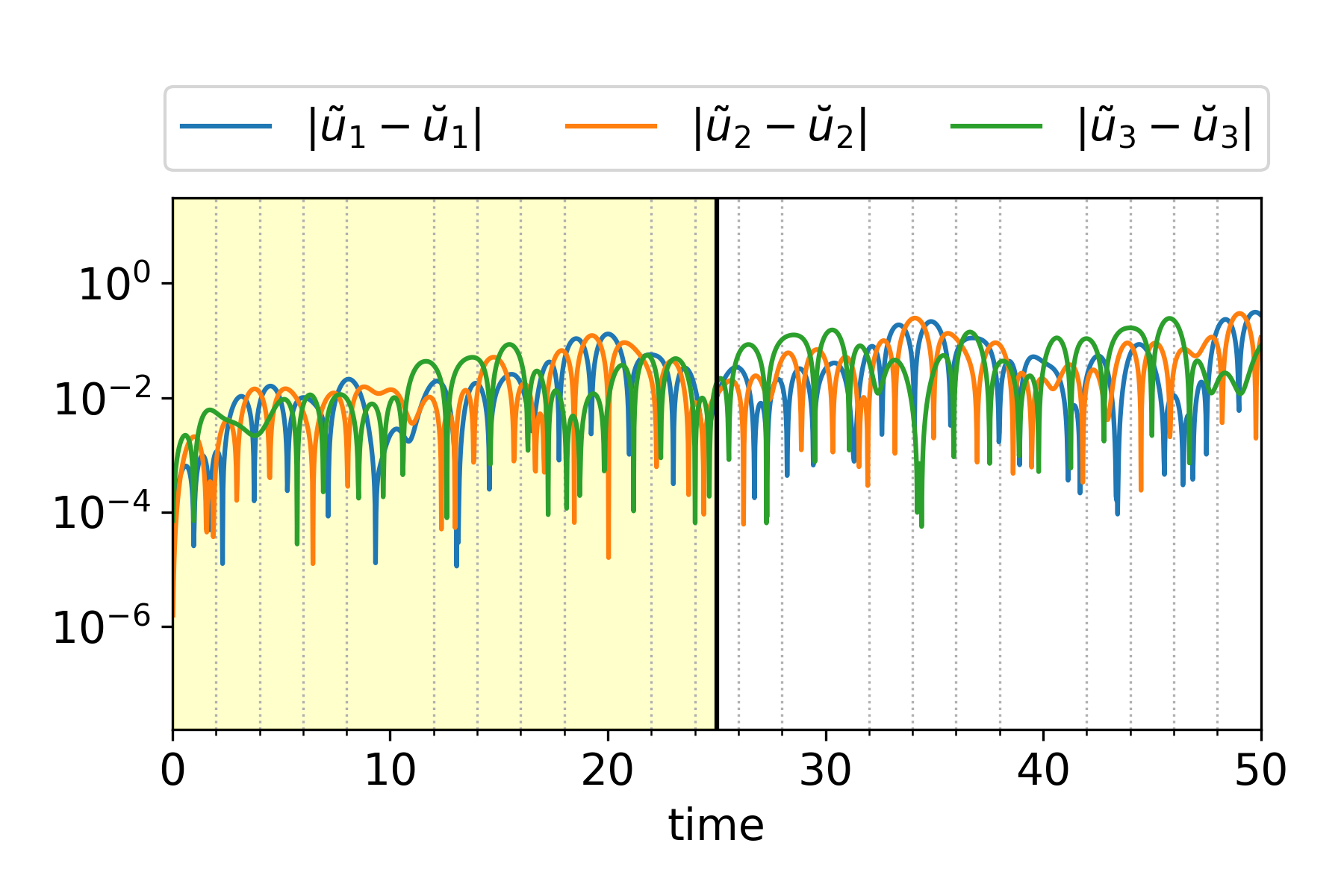}
		\caption{Absolute error}
	\end{subfigure}
	\caption{ZK:  a comparison of the reduced coefficients obtained by the full and the reduced-order models.  }

	\label{fig:zk-rcom}
\end{figure}

\begin{figure}[h]
	\centering
	\begin{subfigure}{0.45\textwidth}
		\includegraphics[width=1\linewidth]{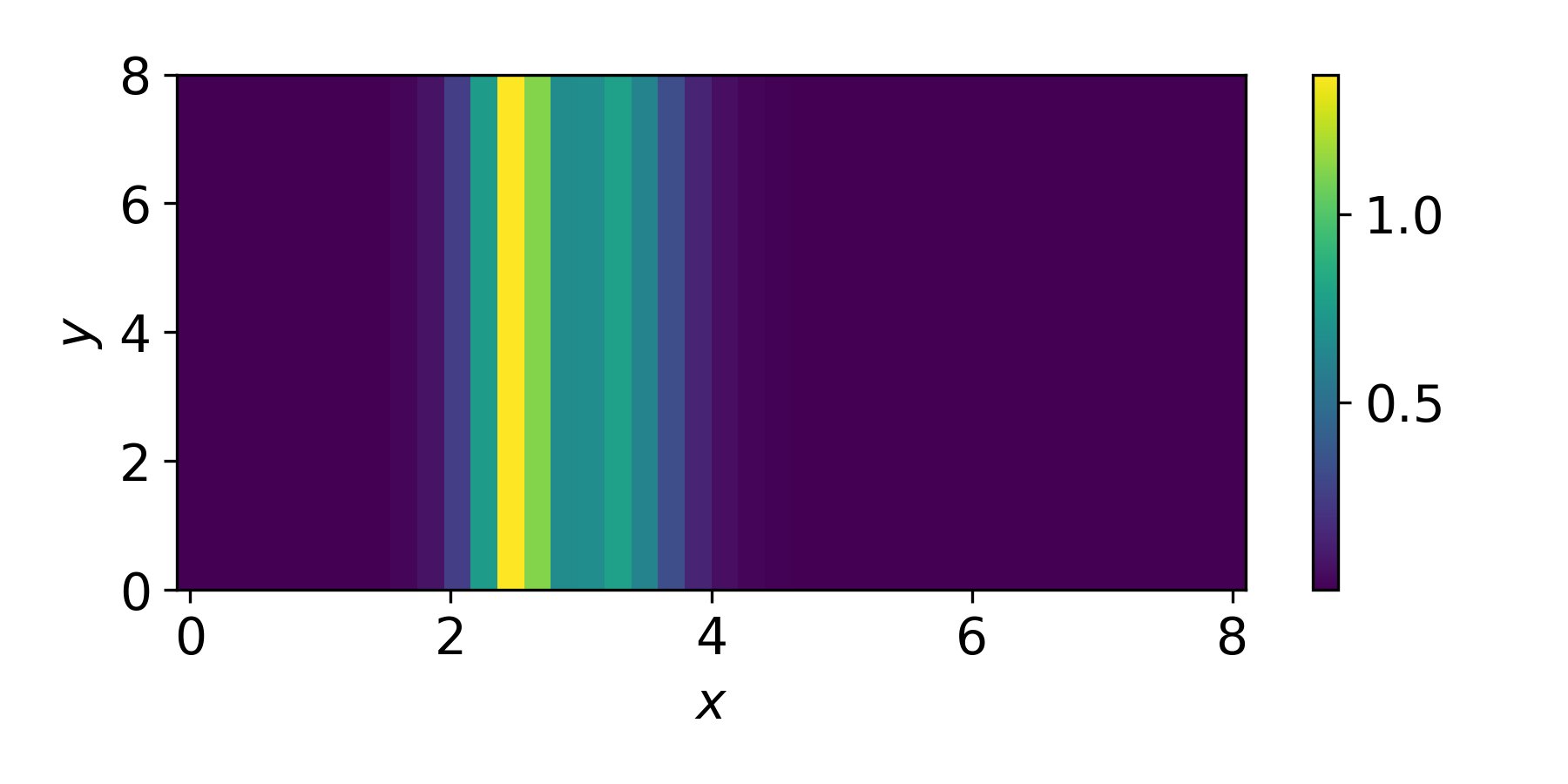}
		\caption{Initial state}	
	\end{subfigure}
	\begin{subfigure}{0.45\textwidth}
		\includegraphics[width=1\linewidth]{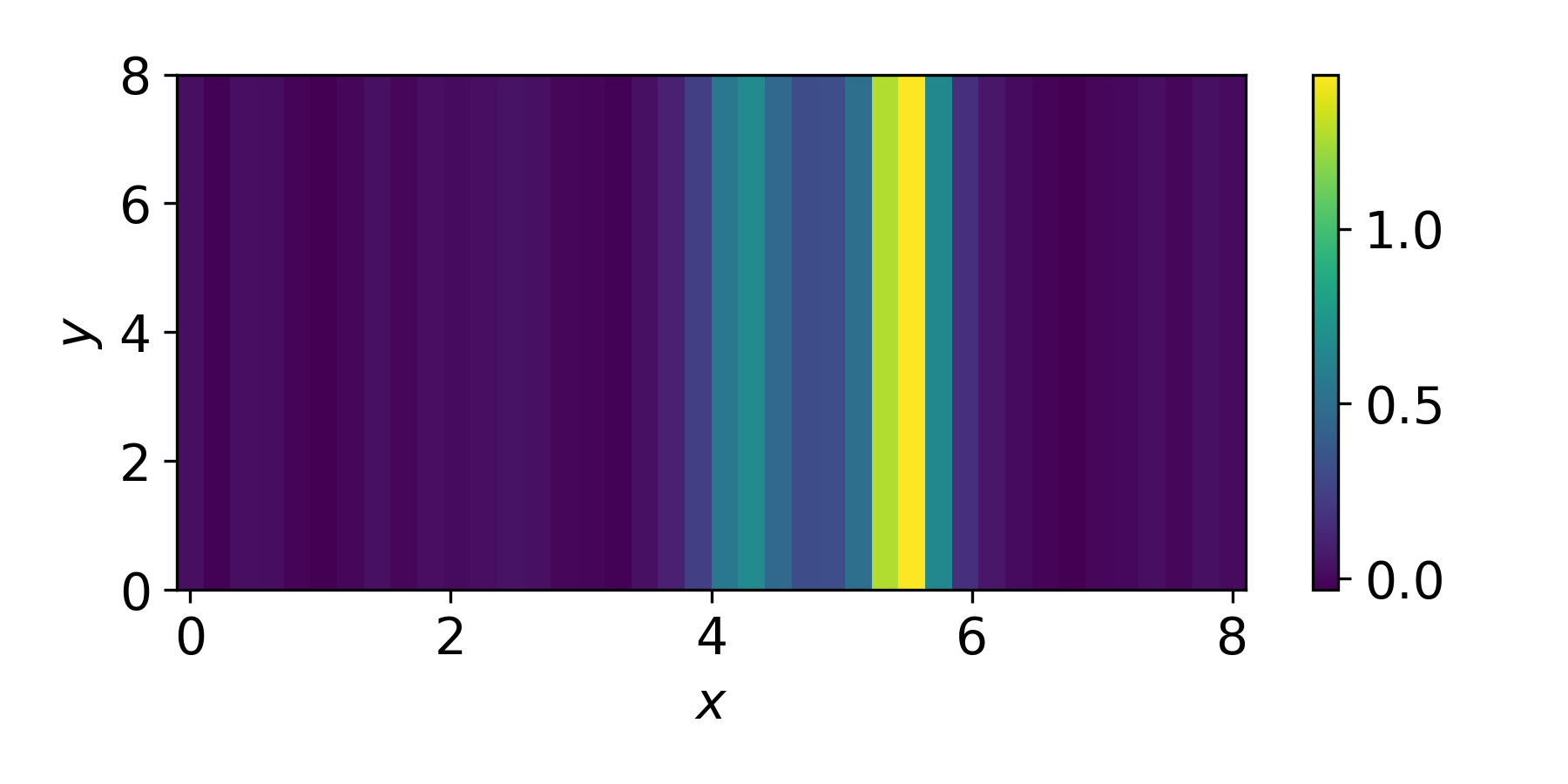}
		\caption{Final state}	
	\end{subfigure}
	\begin{subfigure}{0.45\textwidth}
		\includegraphics[width=1\linewidth]{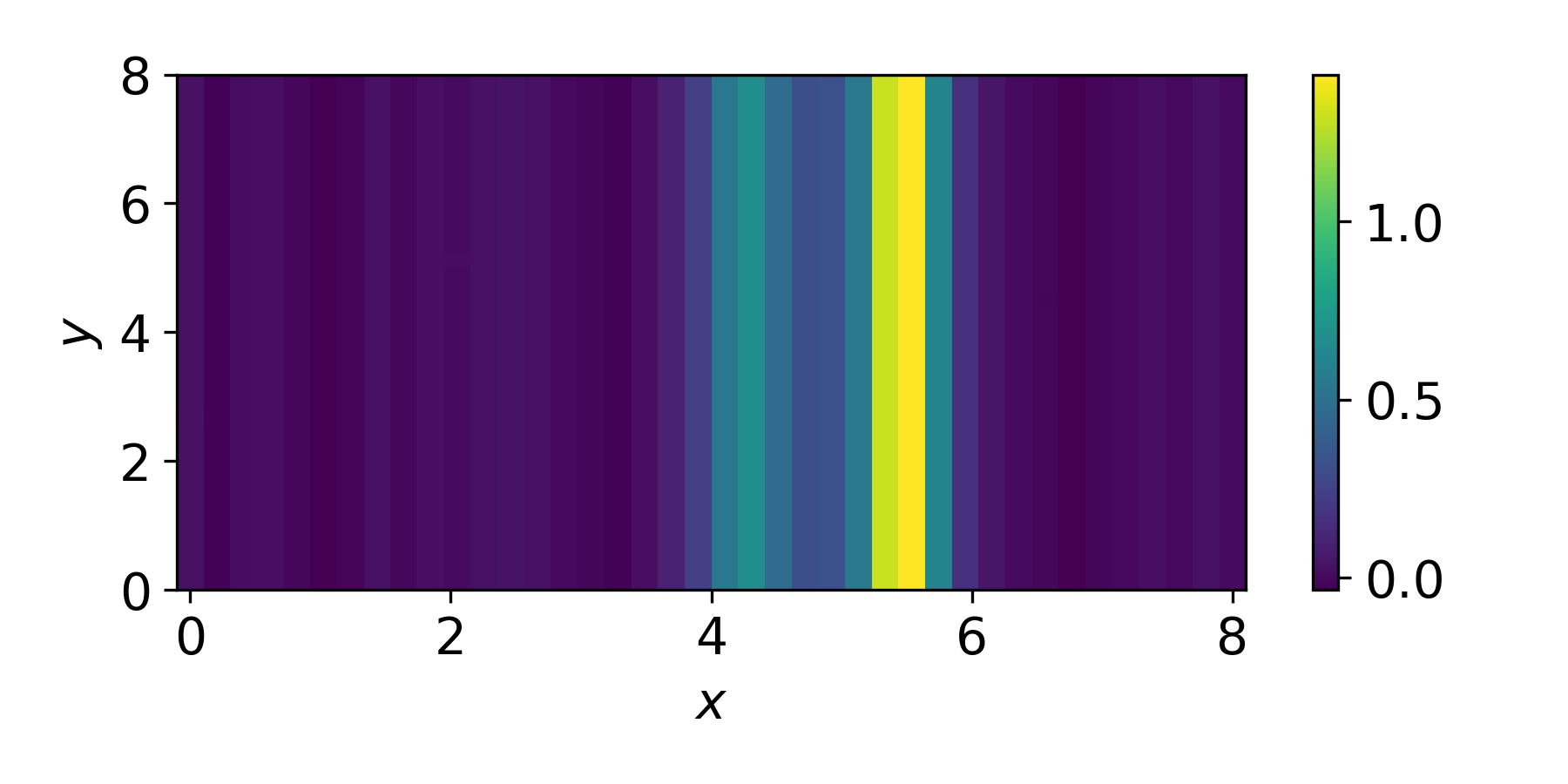} 
		\caption{OpInf final state}
	\end{subfigure}
	\begin{subfigure}{0.45\textwidth}
		\includegraphics[width=1\linewidth]{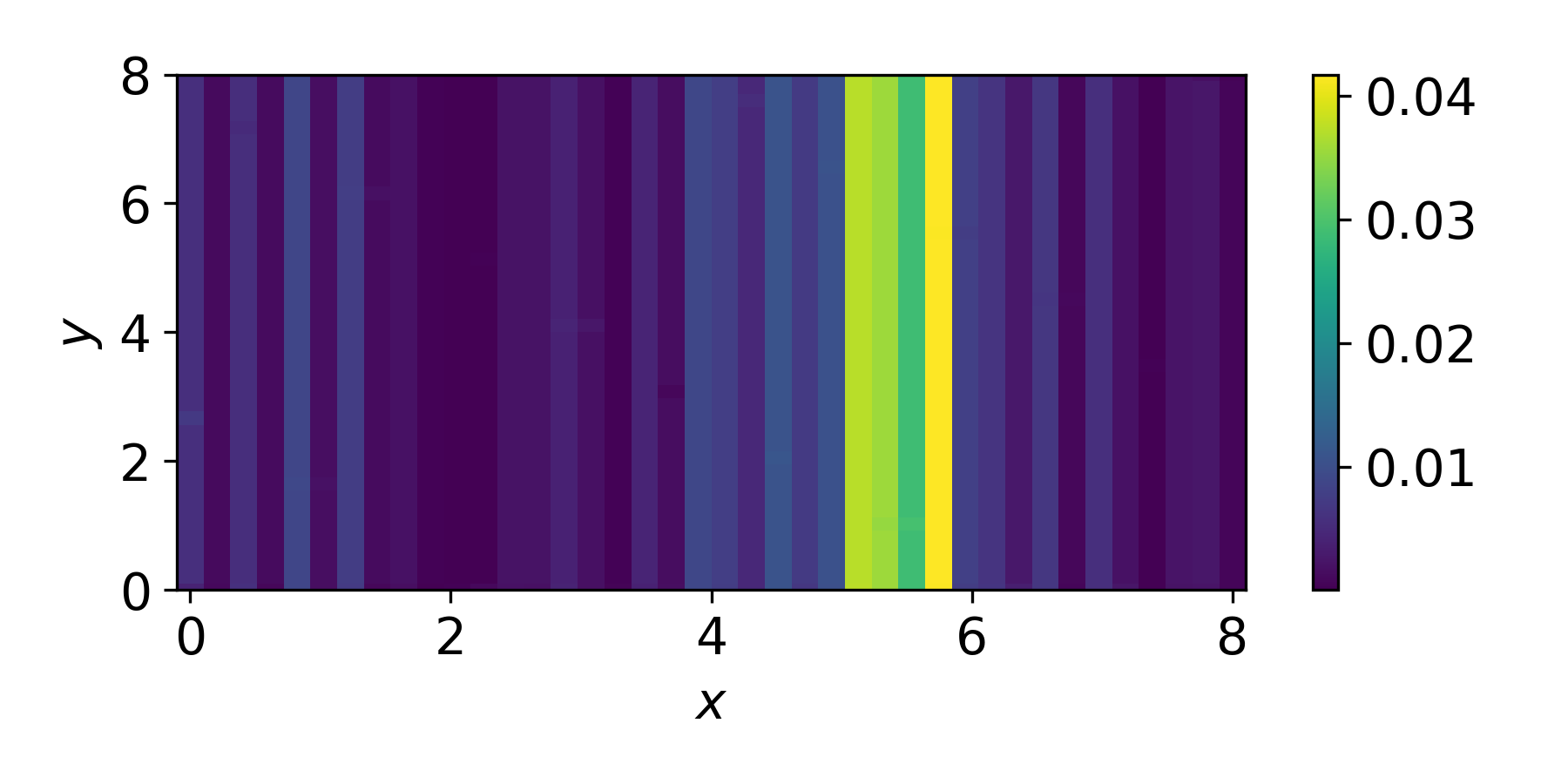}
		\caption{Absolute error}
	\end{subfigure}
	\caption{ZK:   a comparison between the full and reduced-order models in full model dimension at the final time $T=50$.  }
	\label{fig:zk-fcom}
\end{figure}

 \section{Conclusions}\label{sec:conc}

In this paper, we have proposed an operator inference (OpInf) framework by exploiting the idea of global energy preserving PDEs given in multi-symplectic formulations. We have proved that the semi-discrete models obtained by the proposed method preserve the multi-symplectic conservation law and the spatially discrete local energy conservation law. We have demonstrated the efficiency and generality of the proposed method with several numerical examples. We have tested the generality of our method by testing it outside the training interval, which shows that the proposed method leads to robust models outside the training interval. In our future work, we would like to investigate nonlinear transformations by using nonlinear auto-encoders to construct non-intrusive reduced-order models that inherit the multi-symplectic formulations. 

\section*{Acknowledgment}

\section*{Funding Statement}
This work was supported by the German Research Foundation (DFG) Research Training Group 2297 ``MathCoRe'', Magdeburg. 

\section*{Data Availability}

The codes will be available upon publication.

\addcontentsline{toc}{section}{References}
\bibliographystyle{abbrv}
\bibliography{ref}

\end{document}